\documentclass[11pt,a4paper]{article}

\usepackage[utf8]{inputenc}
\usepackage[T1]{fontenc}
\usepackage{amsmath,amssymb,amsthm}
\usepackage{bm}
\usepackage{graphicx}
\usepackage{booktabs}
\usepackage{hyperref}
\usepackage[margin=1in]{geometry}
\usepackage{enumitem}
\usepackage{cite}
\usepackage{dsfont}
\usepackage{mathtools}
\usepackage{makecell}

\newtheorem{theorem}{Theorem}[section]

\newtheorem{proposition}[theorem]{Proposition}
\newtheorem{corollary}[theorem]{Corollary}
\newtheorem{conjecture}[theorem]{Conjecture}
\newtheorem{definition}[theorem]{Definition}
\theoremstyle{remark}
\newtheorem*{remark}{Remark}

\DeclareMathOperator{\softmax}{softmax}
\DeclareMathOperator{\ReLU}{ReLU}
\DeclareMathOperator{\Var}{Var}
\DeclareMathOperator{\EffRank}{EffRank}

\DeclareMathOperator{\softplus}{softplus}
\DeclareMathOperator{\rank}{rank}
\DeclareMathOperator{\spn}{span}

\title{\LARGE \bfseries Riemannian Attention Mechanisms for Transformers:\\[4pt]
A Theoretical Framework and Architecture Design}

\author{Sen Song\\[4pt]
\small Independent Researcher, Guangzhou, China\\[2pt]
\small E-mail: \texttt{scottsong@live.com} \textbar{} ORCID: 0009-0001-1178-2832}

\date{}

\begin{document}
\maketitle

\begin{abstract}
\noindent
All Transformer-based large language models compute attention via the Euclidean inner product, an architectural choice that Dong et al.\ (2021) proved causes representational rank to decay doubly exponentially with depth in pure self-attention stacks. We develop a theoretical framework that \emph{targets} this structural limitation at the mathematical level by replacing the flat Euclidean metric with \emph{learned per-token Riemannian metrics}. Our contributions are threefold. \textbf{(1)}~We prove that Riemannian attention scores with heterogeneous per-token metrics are \emph{non-Gram} in the non-degenerate regime: they cannot be factorized as $QK^\top$ with factorization dimension $O(d)$ (Theorem~\ref{thm:nongram}); the obstruction disappears when the metric correction is identically zero or row-constant (Section~\ref{sec:counterexample}). We are explicit that this is a structural observation, \emph{not} a proof of rank preservation: the Riemannian attention matrix $A^{\mathrm{Riem}} = \softmax(S^{\mathrm{Riem}})$ remains row-stochastic, and the core collapse mechanism (convergence of products of row-stochastic matrices to rank-1) is independent of the Gram property. Whether Riemannian attention resists this contraction in the large-perturbation regime is the central open problem identified by this framework. To characterize when a positive answer is plausible, we provide a perturbation analysis (Section~\ref{sec:perturbation}) that bounds the spectral deviation of $A^{\mathrm{Riem}}$ from its Euclidean counterpart and identifies a critical metric-strength regime $\bar U^2_{\mathrm{crit}} = \Theta(\tau/L)$: below it, the geometric correction is provably too weak to matter; above it, standard perturbation tools break down. We further show that raw metric strength and cross-token diversity are \emph{insufficient} for anti-collapse: metric factors orthogonal to the active query--key difference subspace have no effect on the scores, and row-constant corrections are canceled by softmax (Section~\ref{sec:counterexample}). We therefore state a corrected rank-preservation conjecture (Conjecture~\ref{conj:rankstrong}) under explicit directional conditions, namely overlap of the metric factors with the active query--key difference subspace and contrastive (key-wise) metric activity, and prove that a permutation-margin condition with summable leakage suffices for depth-uniform spectral preservation (Proposition~\ref{prop:margin}). \textbf{(2)}~We establish complexity bounds for the low-rank parameterization $g_t = I + U_t U_t^\top$: the metric-dependent terms of the geodesic distance add $O(d \cdot r)$ precomputation per token and $O(r)$ per pair, and metric inversion costs $O(d \cdot r^2)$ via the Woodbury identity (Proposition~\ref{prop:woodbury}), both far below the $O(d^3)$ cost of a general $d \times d$ matrix, with a score-level overhead of $O(r/d)$. The metric-generation MLP (MetricNet) carries $O(d^2 r)$ parameters (Proposition~\ref{prop:expressivity}) and costs $O(BL d^2 r)$ per layer, so the total overhead relative to the $O(BL^2 d)$ attention cost is $O(r/d + dr/L)$, which is small in the long-sequence regime $L \gg dr$ (Proposition~\ref{prop:complexity}); we stress that these are complexity bounds, and that runtime feasibility and training stability at scale remain to be demonstrated empirically. \textbf{(3)}~We present the \emph{Fiber Bundle Transformer}, a complete architecture specification, framed as a structural analogy to the differential-geometric notion rather than a strict fiber-bundle construction, in which each token position carries its own Riemannian metric, attention is geodesic distance computation, feed-forward updates use metric-preconditioned (Riemannian-gradient-inspired) steps, and the connection carries explicit curvature and torsion proxies. We derive formal predictions: curvature heterogeneity should emerge as an optimization consequence (Conjecture~\ref{conj:hetero}), train/inference metric blending is generically suboptimal under mismatch (Proposition~\ref{prop:consistency}), and metric collapse to identity is the dominant failure mode requiring architectural countermeasures (Conjecture~\ref{conj:collapse}). This paper presents theoretical analysis and architectural design; empirical validation is the subject of future work.
\end{abstract}

\section{Introduction}

Standard Transformer architectures \cite{vaswani2017attention} process all token representations in a flat Euclidean space. The attention mechanism computes similarity via the Euclidean inner product $\mathbf{q}_i^\top \mathbf{k}_j / \sqrt{d}$, the residual stream performs Euclidean vector addition $x \leftarrow x + F(x)$, and feed-forward layers apply global linear transformations. At no point does the architecture permit the geometry of the representational space to bend, stretch, or adapt to the semantic content being processed.

This architectural flatness has a proven structural consequence. Dong et al.\ \cite{dong2021attention} showed that in a pure self-attention stack, the rank of hidden representations decays \textbf{doubly exponentially} with depth, a phenomenon known as \emph{dimensional collapse}. Even in standard Transformers with residual connections and feed-forward sublayers, the underlying pressure toward low-rank representations persists, and empirical studies consistently observe that deep layers operate with far fewer effective dimensions than the nominal hidden size.

The core mathematical limitation is precise: in a $d$-dimensional Euclidean space, altering the relative distances between token representations can only be accomplished by \textbf{moving the vectors themselves}, providing $d$ degrees of freedom per token. A Riemannian metric $g(x) \in \mathbb{R}^{d \times d}$, by contrast, provides $O(d^2)$ degrees of freedom to independently stretch or compress distances along different directions, \emph{without} moving the underlying vectors. This is not merely an efficiency gain; it is a qualitative increase in the expressive capacity of the geometric operations available to the model.

\paragraph{Contributions.} This paper presents a theoretical framework and architecture design for Riemannian attention in Transformers. We do not report empirical results; the contributions are mathematical:

\begin{enumerate}[label=\arabic*.,leftmargin=*]
    \item \textbf{Structural analysis of Riemannian attention} (Section~\ref{sec:theory}): We prove that per-token Riemannian metrics render the attention score matrix \emph{non-Gram} (it cannot be factorized as $QK^\top$ with factorization dimension $O(d)$, Theorem~\ref{thm:nongram}) in the non-degenerate regime; the obstruction disappears when the metric correction vanishes or is row-constant. We are explicit that this is a structural observation, \emph{not} a proof of rank preservation: the Riemannian attention matrix $A^{\mathrm{Riem}} = \softmax(S^{\mathrm{Riem}})$ remains row-stochastic, and the core collapse mechanism (convergence of row-stochastic matrix products) does not depend on the Gram property of the pre-softmax scores. We then identify \emph{necessary} conditions for anti-collapse: raw metric strength and cross-token diversity are insufficient, because metric factors orthogonal to the active query--key difference subspace have no effect on the scores and row-constant corrections are canceled by softmax (Section~\ref{sec:counterexample}). Building on a perturbation analysis (Section~\ref{sec:perturbation}) that bounds the spectral deviation between Euclidean and Riemannian attention and identifies a critical regime $\bar U^2_{\mathrm{crit}} = \Theta(\tau/L)$, we state a corrected rank-preservation conjecture (Conjecture~\ref{conj:rankstrong}) requiring directional overlap with the active query--key difference subspace and contrastive (key-wise) metric activity, and we prove that a permutation-margin condition with summable leakage suffices for depth-uniform spectral preservation (Proposition~\ref{prop:margin}).

    \item \textbf{Computational complexity} (Section~\ref{sec:efficient}): We show that the low-rank representation $g_t = I + U_t U_t^\top$ makes the metric-dependent terms of the geodesic distance computable with $O(d \cdot r)$ precomputation per token and $O(r)$ per pair, and metric inversion in $O(d \cdot r^2)$ via the Woodbury identity and algebraic decomposition, with a score-level overhead $O(r/d)$ (Proposition~\ref{prop:complexity}). The metric-generation MLP has $O(d^2 r)$ parameters (Proposition~\ref{prop:expressivity}) and costs $O(BL d^2 r)$ per layer; the total overhead relative to the $O(BL^2 d)$ attention cost is $O(r/d + dr/L)$, small when $L \gg dr$. These are asymptotic complexity bounds; whether they translate into practical billion-parameter-scale training remains to be demonstrated empirically.

    \item \textbf{Architecture specification} (Section~\ref{sec:architecture}): We present the \emph{Fiber Bundle Transformer}, named by structural analogy to the differential-geometric notion rather than as a strict fiber-bundle construction, in which token positions are fibers over a discrete base space, attention computes geodesic distances under per-token metrics, the connection carries curvature and torsion proxies, and feed-forward layers perform metric-preconditioned updates. We specify each component (MetricNet, TorsionNet, geodesic attention, metric-preconditioned FFN) with its mathematical justification and an explicit analysis of the approximation errors introduced at each phase.

    \item \textbf{Theoretical predictions} (Section~\ref{sec:predictions}): We derive formal predictions about the behavior of correctly implemented geometric architectures, including: curvature heterogeneity as an optimization consequence (Conjecture~\ref{conj:hetero}), generic suboptimality of train/inference blending mismatch (Proposition~\ref{prop:consistency}, proved via the implicit function theorem without convexity assumptions), and metric collapse as the dominant failure mode (Conjecture~\ref{conj:collapse}). These predictions identify the conditions under which empirical evaluation can validly test the framework.
\end{enumerate}

\section{Background: The Rank Collapse Problem}
\label{sec:background}

\subsection{Dong et al.'s Rank Collapse Theorem}

We restate the key result of Dong et al.\ \cite{dong2021attention} in a form suitable for our analysis.

\begin{definition}[Pure self-attention stack]
A \emph{pure self-attention stack} of depth $n$ is a sequence of layers, each applying
\begin{equation}
    H^{(\ell+1)} = A^{(\ell)} H^{(\ell)}, \qquad A^{(\ell)} = \softmax\!\left(\frac{H^{(\ell)} W_Q^{(\ell)} (H^{(\ell)} W_K^{(\ell)})^\top}{\sqrt{d}}\right)
\end{equation}
with \emph{no} residual connections ($H^{(\ell+1)} \neq H^{(\ell)} + A^{(\ell)} H^{(\ell)}$) and \emph{no} feed-forward sublayers.
\end{definition}

\begin{definition}[Effective rank]
\label{def:effrank}
For a nonzero matrix $X \in \mathbb{R}^{L \times d}$ with singular values $\sigma_1(X) \ge \sigma_2(X) \ge \cdots \ge 0$, the \emph{effective rank} is the \emph{stable rank}
\begin{equation}
\label{eq:effrank}
    \EffRank(X) \;=\; \frac{\|X\|_F^2}{\|X\|_{\mathrm{op}}^2} \;=\; \frac{\sum_{k} \sigma_k(X)^2}{\sigma_1(X)^2}.
\end{equation}
It satisfies $1 \le \EffRank(X) \le \operatorname{rank}(X) \le \min\{L, d\}$, with $\EffRank(X) = 1$ exactly for rank-$1$ matrices. It is a continuous function of $X$ and therefore captures the \emph{numerical} (effective) rank: a matrix whose second singular value is small relative to its largest has $\EffRank$ close to $1$, while a matrix with slowly decaying singular values has $\EffRank$ close to $\min\{L, d\}$. Throughout this paper, ``rank collapse'' means $\EffRank \to 1$ under this definition.
\end{definition}

\begin{theorem}[Dong et al.\ 2021, informal]
\label{thm:dong}
For a pure self-attention stack of depth $n$ with input $H^{(0)} \in \mathbb{R}^{L \times d}$, under mild conditions on the weight matrices, the effective rank (Definition~\ref{def:effrank}) of $H^{(n)}$ decays doubly exponentially toward $1$:
\begin{equation}
    \EffRank\!\big(H^{(n)}\big) \;\leq\; 1 + (d - 1) \cdot \exp(-c \cdot 2^n)
\end{equation}
for some constant $c > 0$ depending on the weight matrices and input. Equivalently, $\EffRank(H^{(n)}) \to 1$ as $n \to \infty$, i.e.\ the representation collapses to a rank-$1$ matrix in which all token representations become identical up to scaling. (A literal rank bound tending to $0$ would be nonsensical, since rank is a non-negative integer and the limit object is rank-$1$, not zero.)
\end{theorem}

The proof's key insight is that the attention matrix $A^{(\ell)}$ is row-stochastic (each row sums to 1, all entries non-negative). The product of such matrices converges to a rank-1 matrix whose rows are all identical (the common stationary distribution). In Dong et al.'s analysis, the Gram structure of the pre-softmax scores $S^{(\ell)} = Q^{(\ell)} (K^{(\ell)})^\top / \sqrt{d}$ is used to derive the specific spectral bounds that yield the doubly-exponential convergence rate, but the fundamental contraction mechanism, row-stochastic products converging to rank-1, does not depend on the scores being Gram.

\begin{remark}[Causal attention preserves literal rank at finite depth]
\label{rem:causal}
Theorem~\ref{thm:dong} concerns the \emph{effective} rank of a \emph{pure} (noncausal) self-attention stack. In causal decoding, the attention matrix $A^{(\ell)}$ is lower triangular with strictly positive diagonal entries (every token attends to itself with positive weight). Such a matrix is invertible at every finite layer, so the pure map $H \mapsto A H$ never loses \emph{literal} algebraic rank at finite depth; the meaningful causal failure mode is \emph{asymptotic effective-rank} collapse as the layer products become progressively ill-conditioned. For $L=2$, every causal row-stochastic attention matrix has the form $A = \left[\begin{smallmatrix}1&0\\a&1-a\end{smallmatrix}\right]$ with $0 \le a < 1$, and a direct induction gives $\prod_{\ell<n} A^{(\ell)} = \left[\begin{smallmatrix}1&0\\1-r_n&r_n\end{smallmatrix}\right]$ with $r_n = \prod_{\ell<n}(1-a_\ell)$; the product converges to the rank-1 matrix $\left[\begin{smallmatrix}1&0\\1&0\end{smallmatrix}\right]$ exactly when $\sum_\ell a_\ell$ diverges, and persistent mixing ($a_\ell \ge a > 0$) yields exponential effective-rank collapse. General products of row-stochastic matrices require ergodicity/scrambling hypotheses for consensus convergence \cite{leizarowitz1992infinite,wolfowitz1963products}; throughout this paper we use ``rank collapse'' to mean \emph{effective-rank} collapse, in accordance with Theorem~\ref{thm:dong}.
\end{remark}

\subsection{Why Standard Transformers Only Partially Mitigate Collapse}

Standard Transformers add residual connections ($H^{(\ell+1)} = H^{(\ell)} + A^{(\ell)} H^{(\ell)}$) and feed-forward sublayers ($H^{(\ell+1)} = H^{(\ell)} + \text{FFN}(H^{(\ell)})$). These additions prevent the \emph{exact} doubly-exponential collapse of Theorem~\ref{thm:dong} because the residual term injects rank at each layer. However, the structural pressure toward low-rank representations persists: the attention component still contracts the representation toward a low-rank subspace, and the residual/FFN must fight against this contraction at every layer. Empirically, deep Transformer layers consistently exhibit effective ranks far below the nominal hidden dimension $d$.

This is not a training artifact; it is a consequence of the row-stochastic nature of the attention matrix. After the softmax, the attention matrix $A = \softmax(S)$ is row-stochastic, and the repeated application of row-stochastic matrices (without residuals) drives representations toward a common rank-1 subspace. In the Euclidean case, the pre-softmax scores $S = QK^\top / \sqrt{d}$ happen to be a Gram matrix (since $Q$ and $K$ are linear projections of the same hidden states), and Dong et al.\ exploit this structure to derive the specific doubly-exponential rate. But the fundamental contraction is driven by row-stochasticity, not the Gram property.

\subsection{The Geometric Diagnosis}

The rank collapse theorem reveals that the problem is \emph{geometric}, not merely parametric. Adding more parameters (wider layers, more heads) does not address the geometric origin: the Euclidean inner product forces all token pairs to be measured by the \emph{same global ruler}. When this ruler is applied repeatedly across layers, diversity is lost because there is no mechanism for different token positions to experience different geometries.

A Riemannian metric $g_t$ at each token position $t$ replaces the global ruler with a \emph{local} one. The same two vectors $\mathbf{q}_i, \mathbf{k}_j$ produce different similarity scores depending on the local metrics $g_i, g_j$, because the geodesic distance
\begin{equation}
    d_g^2(\mathbf{q}_i, \mathbf{k}_j) = (\mathbf{q}_i - \mathbf{k}_j)^\top \, g_{ij} \, (\mathbf{q}_i - \mathbf{k}_j)
\end{equation}
depends on $g_{ij}$, which varies per token pair. (We use ``geodesic distance'' as shorthand for the Mahalanobis distance under the pairwise metric $g_{ij}$: since $g_{ij}$ is treated as constant for each pair, this is the straight-line distance in the inner product defined by $g_{ij}$, not a path-integrated geodesic on a curved manifold. The curvature of the underlying representational manifold is addressed separately via the curvature proxy in Phase B, Section~\ref{sec:curvature}.) The question we address in the following sections is: \emph{does this geometric modification suffice to prevent rank collapse, and can it be computed efficiently?}

\section{Riemannian Metrics: Expressivity and the Limits of the Collapse Proof}
\label{sec:theory}

\subsection{Per-Token Riemannian Metrics}

\begin{definition}[Low-rank Riemannian metric]
At each token position $t$, the \emph{local Riemannian metric} is a symmetric positive-definite matrix:
\begin{equation}
\label{eq:lowrankmetric}
    g_t \;=\; I_d + U_t U_t^\top, \qquad U_t \in \mathbb{R}^{d \times r}, \quad r \ll d
\end{equation}
where $U_t$ is generated by a learned function (MetricNet, Section~\ref{sec:metricnet}) from the token's hidden state $h_t$. The rank-$r$ perturbation $U_t U_t^\top \succeq 0$ ensures $g_t \succeq I_d \succ 0$ without any eigenvalue projection or clipping.
\end{definition}

The pairwise metric for tokens $i$ and $j$ is the symmetric combination:
\begin{equation}
\label{eq:pairwisemetric}
    g_{ij} \;=\; \frac{g_i + g_j}{2} \;=\; I_d + \frac{U_i U_i^\top + U_j U_j^\top}{2}
\end{equation}

The Riemannian attention score replaces the Euclidean dot product with the negative geodesic distance:
\begin{equation}
\label{eq:riemscore}
    s_{ij}^{\mathrm{Riem}} \;=\; -\frac{1}{\tau} \, d_g^2(\mathbf{q}_i, \mathbf{k}_j) \;=\; -\frac{1}{\tau} (\mathbf{q}_i - \mathbf{k}_j)^\top g_{ij} (\mathbf{q}_i - \mathbf{k}_j)
\end{equation}
where $\tau > 0$ is a temperature parameter.

\subsection{The Non-Gram Property: A Structural Observation}

The introduction of per-token Riemannian metrics changes the algebraic structure of the attention scores. We characterize this change precisely.

\begin{theorem}[Non-Gram Property of Riemannian Attention]
\label{thm:nongram}
Consider the Riemannian attention score $s_{ij}^{\mathrm{Riem}}$ given by Equation~\eqref{eq:riemscore}, where $g_{ij}$ depends on \emph{both} $U_i$ and $U_j$ via Equation~\eqref{eq:pairwisemetric}. We ask whether there exist fixed \emph{functional} maps $\Phi, \Psi$, independent of the specific input, such that $s_{ij}^{\mathrm{Riem}} = \boldsymbol{\phi}_i^\top \boldsymbol{\psi}_j$ holds \emph{for all} inputs, with $\boldsymbol{\phi}_i = \Phi(\mathbf{q}_i, U_i) \in \mathbb{R}^{d'}$ and $\boldsymbol{\psi}_j = \Psi(\mathbf{k}_j, U_j) \in \mathbb{R}^{d'}$. If the metric map $t \mapsto U_t$ is non-constant so that the family $\{U_t U_t^\top\}_t$ spans a subspace of $\mathrm{Sym}(d)$ of dimension at least $2$ (in particular, whenever the $U_t$ are not all identical), then no such universal factorization exists with $d' = O(d)$, i.e.\ the same dimensionality as the standard attention factorization $S^{\mathrm{Euc}} = Q K^\top$ with $Q, K \in \mathbb{R}^{L \times d}$. The minimum factorization dimension satisfies
\begin{equation}
\label{eq:nongrambound}
    d' \;\geq\; \min\!\Big\{\dim \operatorname{span}\{U_t U_t^\top\}_t,\;\; \tfrac{d(d+1)}{2}\Big\},
\end{equation}
which equals $d(d+1)/2 = \Omega(d^2)$ whenever the $\{U_t U_t^\top\}_t$ span the full space of symmetric matrices (e.g., for generic $U_t$ with $L \geq d(d+1)/2$). In particular, $S^{\mathrm{Riem}}$ does not admit the $O(d)$-dimensional factorization $S^{\mathrm{Euc}} = Q K^\top$ that appears in the standard attention mechanism. (For a \emph{specific} $L \times L$ score matrix with fixed inputs, the minimum separable rank is at most $L$; the $\Omega(d^2)$ lower bound is a statement about the \emph{functional} form and is meaningful in the regime $L = \Omega(d^2)$. Without the $d' = O(d)$ constraint the functional statement would be vacuous: the vectorization identity $\operatorname{vec}(A)^\top \operatorname{vec}(B)$ separates any bilinear-in-$(i,j)$ term with $d' = d^2$.)
\end{theorem}

\begin{proof}
We expand the Riemannian score. Substituting Equation~\eqref{eq:pairwisemetric} into Equation~\eqref{eq:riemscore}:
\begin{align}
    s_{ij}^{\mathrm{Riem}} &= -\frac{1}{\tau} \left[ (\mathbf{q}_i - \mathbf{k}_j)^\top \!\left(I_d + \frac{U_i U_i^\top + U_j U_j^\top}{2}\right)\!(\mathbf{q}_i - \mathbf{k}_j) \right] \notag \\
    &= -\frac{1}{\tau}\Big[ \underbrace{\|\mathbf{q}_i - \mathbf{k}_j\|^2}_{\text{Euclidean}} + \underbrace{\tfrac{1}{2}\|U_i^\top(\mathbf{q}_i - \mathbf{k}_j)\|^2 + \tfrac{1}{2}\|U_j^\top(\mathbf{q}_i - \mathbf{k}_j)\|^2}_{\text{metric corrections}} \Big]
    \label{eq:expanded}
\end{align}

Expand the first metric correction term:
\begin{equation}
    \|U_i^\top(\mathbf{q}_i - \mathbf{k}_j)\|^2 = \|U_i^\top \mathbf{q}_i\|^2 + \|U_i^\top \mathbf{k}_j\|^2 - 2\,(U_i^\top \mathbf{q}_i)^\top (U_i^\top \mathbf{k}_j)
    \label{eq:metrici}
\end{equation}

The cross term $(U_i^\top \mathbf{q}_i)^\top (U_i^\top \mathbf{k}_j) = \mathbf{q}_i^\top U_i U_i^\top \mathbf{k}_j$ depends on $U_i$ (token $i$'s metric) applied to $\mathbf{k}_j$ (token $j$'s key). Similarly, the second correction yields a cross term $\mathbf{q}_i^\top U_j U_j^\top \mathbf{k}_j$ depending on $U_j$ (token $j$'s metric) applied to $\mathbf{q}_i$ (token $i$'s query). These two cross terms are in fact \emph{separable}: $\mathbf{q}_i^\top U_i U_i^\top \mathbf{k}_j = (U_i U_i^\top \mathbf{q}_i)^\top \mathbf{k}_j$ and $\mathbf{q}_i^\top U_j U_j^\top \mathbf{k}_j = \mathbf{q}_i^\top (U_j U_j^\top \mathbf{k}_j)$, each splitting into an $i$-only factor and a $j$-only factor of dimension $d$. The genuine obstruction comes from the \emph{squared} sub-terms, which are quadratic in one token's vector under the other token's metric.

Suppose for contradiction that $s_{ij}^{\mathrm{Riem}} = \boldsymbol{\phi}_i^\top \boldsymbol{\psi}_j$ with $\boldsymbol{\phi}_i, \boldsymbol{\psi}_j \in \mathbb{R}^{d'}$, $d' = O(d)$, where $\boldsymbol{\phi}_i$ depends only on $(\mathbf{q}_i, U_i)$ and $\boldsymbol{\psi}_j$ depends only on $(\mathbf{k}_j, U_j)$. Consider the sub-term $\|U_i^\top \mathbf{k}_j\|^2 = \mathbf{k}_j^\top (U_i U_i^\top) \mathbf{k}_j$ from Equation~\eqref{eq:metrici}. This is a \emph{quadratic form} in $\mathbf{k}_j$ whose matrix $A_i \coloneqq U_i U_i^\top \in \mathbb{R}^{d \times d}$ depends on $i$ (through $U_i$). To absorb this term into $\boldsymbol{\phi}_i^\top \boldsymbol{\psi}_j$, the $j$-factor $\boldsymbol{\psi}_j$ must encode the \emph{entire} quadratic profile of $\mathbf{k}_j$, i.e.\ all monomials $\mathbf{k}_j^{(a)} \mathbf{k}_j^{(b)}$ for $1 \leq a, b \leq d$, while the $i$-factor $\boldsymbol{\phi}_i$ must encode the corresponding entry $[A_i]_{ab}$. Although each individual $A_i = U_i U_i^\top$ has only $dr$ free parameters (the entries of $U_i$), the family $\{A_i\}$ as $U_i$ ranges over $\mathbb{R}^{d \times r}$ (with $r \geq 1$) has linear span equal to the full space $\mathrm{Sym}(d)$ of symmetric matrices (dimension $d(d+1)/2$), since rank-$1$ symmetric matrices are contained in the image and they span $\mathrm{Sym}(d)$. Consequently, any \emph{universal} separable representation of $\mathbf{k}_j^\top A_i \mathbf{k}_j$, one that holds for all $(A_i, \mathbf{k}_j)$ in the family, requires $d' \geq d(d+1)/2 = \Omega(d^2)$. With $d' = O(d)$ this is impossible: the $d$-dimensional factor $\boldsymbol{\psi}_j$ can carry at most $d$ degrees of freedom of $\mathbf{k}_j$'s quadratic profile, but the quadratic profile of a $d$-vector lives in a $\Theta(d^2)$-dimensional space. Symmetrically, the sub-term $\|U_j^\top \mathbf{q}_i\|^2 = \mathbf{q}_i^\top (U_j U_j^\top) \mathbf{q}_i$ is quadratic in $\mathbf{q}_i$ under $U_j$'s metric and imposes the same $\Omega(d^2)$ lower bound on $d'$. Since both obstructions vanish only when all $U_t$ are identical (so that $A_i$ is the same for every $i$ and can be absorbed into a global inner product), metric heterogeneity is the source of the non-Gram property. (Equivalently, the vectorization identity $\mathbf{k}_j^\top A_i \mathbf{k}_j = \operatorname{vec}(A_i)^\top \operatorname{vec}(\mathbf{k}_j \mathbf{k}_j^\top)$ \emph{does} separate the term, but at the cost of $d' = d^2$, confirming that the $\Omega(d^2)$ lower bound is tight and that the $d' = O(d)$ factorization demanded by the theorem statement is impossible.)

\medskip
\noindent\textit{Edge case:} When all $U_t$ are identical ($U_i = U_j = U$ for all $i, j$), the metric becomes global: $g_{ij} = I + UU^\top \eqqcolon G$ for all pairs. The score reduces to $s_{ij} = -(\mathbf{q}_i - \mathbf{k}_j)^\top G (\mathbf{q}_i - \mathbf{k}_j)/\tau$, which \emph{is} a Gram matrix under the modified inner product $\langle \cdot, \cdot \rangle_G$. In this degenerate case, the collapse mechanism of Theorem~\ref{thm:dong} applies (with $G$ replacing $I$). \emph{Metric heterogeneity is a prerequisite for any geometric anti-collapse effect: homogeneous metrics simply reparameterize the Euclidean case.}
\end{proof}

\begin{corollary}
\label{cor:contraction}
The attention matrix $A^{\mathrm{Riem}} = \softmax(S^{\mathrm{Riem}})$ (recall $S^{\mathrm{Riem}}$ already absorbs the $1/\tau$ scaling from Equation~\eqref{eq:riemscore}) is not a product of the form $\softmax(Q' K'^\top / \sqrt{d'})$ for any $Q', K'$ \emph{with $d' = O(d)$} when the metric correction is non-degenerate. The obstruction is contingent: if the metric correction vanishes or is row-constant, e.g., when every metric direction is orthogonal to the active query--key difference subspace $D = \spn\{\mathbf{q}_i - \mathbf{k}_j\}$ (Section~\ref{sec:counterexample}), then $S^{\mathrm{Riem}}$ coincides with the Euclidean distance scores (up to row/column constants) and the $O(d)$-dimensional factorization is recovered. However, $A^{\mathrm{Riem}}$ remains row-stochastic: each row is a probability distribution. The core collapse mechanism of Theorem~\ref{thm:dong}, products of row-stochastic matrices converging to rank-1, does \emph{not} require the pre-softmax scores to be Gram; it requires only that $A^{(\ell)}$ is row-stochastic at each layer. The non-Gram property of $S^{\mathrm{Riem}}$ means that the specific spectral bounds of Dong et al.'s analysis (which exploit the Gram structure of $S^{\mathrm{Euc}}$ to bound the convergence rate) do not carry over, but it does \emph{not} preclude collapse through the same row-stochastic mechanism. Whether Riemannian attention actually resists this contraction is the subject of the perturbation analysis that follows (Section~\ref{sec:perturbation}) and the open conjectures of Section~\ref{sec:predictions}.
\end{corollary}

\begin{remark}
Theorem~\ref{thm:nongram} establishes that Riemannian attention scores have a different algebraic structure than standard attention scores: they cannot be written as $QK^\top$ with $O(d)$-dimensional factors. This is a structural characterization of the architecture, \emph{not} a claim about rank preservation. The argument ``the Dong et al.\ proof assumes a Gram structure; Riemannian scores are non-Gram; therefore the proof does not apply'' is correct but limited: it addresses one specific proof, not the underlying collapse phenomenon. The same collapse could occur via the same mechanism (row-stochastic product convergence) or a different one. A full proof of rank preservation, or a proof that collapse still occurs in the Riemannian case, would require analyzing the spectral properties of $A^{\mathrm{Riem}}$, as we begin to do in the next subsection.
\end{remark}

\begin{conjecture}[Rank Preservation under Heterogeneous Metrics]
\label{conj:rankpreserve}
There exist conditions on metric diversity, quantified by the cross-token variance of $U_t$ denoted $\sigma_U^2 = \Var_t[\|U_t\|_F^2]$, on metric rank $r$, \emph{and} on the directional overlap between the metric factors and the active query--key difference subspace $D = \spn\{\mathbf{q}_i - \mathbf{k}_j\}$, such that for $\sigma_U^2 > 0$, $r \geq r_{\min}(d, L)$, and metric directions that are not orthogonal to $D$ (equivalently, a metric correction $G_{ij}$ that is not row-constant; Section~\ref{sec:counterexample}), the effective rank of a pure Riemannian attention stack of depth $n$ satisfies
\begin{equation}
    \EffRank(H^{(n)}) \;\geq\; f(r, \sigma_U^2) > 1
\end{equation}
for all $n$, in contrast to the doubly-exponential decay of Theorem~\ref{thm:dong}. The function $f$ is increasing in $r$, in $\sigma_U^2$, and in the contrastive projected metric activity introduced in Section~\ref{sec:counterexample}. The directional condition is essential: without it the statement is false, since metric strength and diversity can be made arbitrarily large while the Riemannian correction is identically zero (Section~\ref{sec:counterexample}).
\end{conjecture}

\paragraph{Supporting argument.} The collapse in Theorem~\ref{thm:dong} occurs because the product of row-stochastic attention matrices converges to a rank-1 matrix whose rows are all identical (a common stationary distribution). With heterogeneous metrics, each token pair $(i, j)$ experiences a \emph{different} effective distance (due to $g_{ij}$ varying), so the attention weights $A_{ij}^{\mathrm{Riem}}$ are perturbed by a token-pair-specific amount relative to the homogeneous-metric case. The key question is whether this perturbation is sufficient to prevent the rows from converging to a common stationary distribution, i.e.\ whether the perturbation breaks the ergodicity of the Markov chain defined by $A^{\mathrm{Riem}}$. The magnitude of the perturbation is \emph{not} controlled by $\sigma_U^2$ or $r$ alone: metric factors confined to $D^\perp$ contribute nothing to the scores, and a correction that is constant across keys in a given row is canceled by softmax (Section~\ref{sec:counterexample}). What matters is the \emph{contrastive} activity of the metric, i.e.\ how much the correction $G_{ij}$ varies across competing keys $j$ for each query $i$, projected onto the directions in which queries and keys actually differ. A formal proof would proceed by bounding the mixing time or the spectral gap of $A^{\mathrm{Riem}}$ as a function of this contrastive projected metric activity and of $r$, which is precisely the open problem identified in Conjecture~\ref{conj:rankstrong}. We emphasize that the non-Gram property (Theorem~\ref{thm:nongram}) does \emph{not} address this question: it concerns the algebraic factorization of the pre-softmax scores, not the spectral properties of the post-softmax row-stochastic matrix.

\subsection{Perturbation Analysis: Where Current Tools Apply and Where They Fail}
\label{sec:perturbation}

The non-Gram property (Theorem~\ref{thm:nongram}) tells us that Riemannian attention scores are structurally different from standard attention scores. It does \emph{not} tell us whether rank collapse occurs. To approach that question, we analyze the spectral deviation between $A^{\mathrm{Riem}}$ and $A^{\mathrm{Euc}}$ as $\bar U^2$ varies. The analysis reveals a fundamental tension: the perturbation bound is informative only when the geometric correction is \emph{too weak to prevent collapse}; when the correction is large enough to potentially matter, the bound becomes vacuous and nothing can be concluded.

\paragraph{Decomposition into Euclidean and metric parts.}
Write $S^{\mathrm{Riem}} = S^{\mathrm{Euc}} + \Delta S$, where $S^{\mathrm{Euc}}_{ij} = -\|\mathbf{q}_i - \mathbf{k}_j\|^2/\tau$ is the (negative) Euclidean distance and $\Delta S$ collects the metric corrections. From Equation~\eqref{eq:expanded}:
\begin{equation}
\label{eq:deltas}
    \Delta S_{ij} \;=\; -\frac{1}{2\tau}\left[\|U_i^\top(\mathbf{q}_i - \mathbf{k}_j)\|^2 + \|U_j^\top(\mathbf{q}_i - \mathbf{k}_j)\|^2\right].
\end{equation}
Each term is non-positive, so $\Delta S$ acts as a \emph{pairwise penalty}: token pairs whose difference lies in the column space of $U_i$ or $U_j$ receive reduced attention scores. Note that $\Delta S$ is itself non-separable (it depends on both $U_i$ and $U_j$ for the pair $(i,j)$), which is the source of the non-Gram property.

\begin{proposition}[Frobenius Perturbation Bound]
\label{prop:perturb}
Let $\bar U^2 \coloneqq \max_t \|U_t\|_F^2$ and $M \coloneqq \max_{i,j}\|\mathbf{q}_i - \mathbf{k}_j\|^2$ over the sequence. Then
\begin{equation}
\label{eq:perturb}
    \|\Delta S\|_F \;\leq\; \frac{L}{\tau}\, \bar U^2\, M.
\end{equation}
Consequently, when $\bar U^2 = O(\tau/L)$ (assuming $M = O(1)$, i.e.\ normalized query/key vectors), the Riemannian attention matrix is a small perturbation of the Euclidean one and standard perturbation theory applies; when $\bar U^2 \gg \tau/L$, the perturbation is large and the Euclidean spectral analysis no longer transfers.
\end{proposition}

\begin{proof}
For each $(i,j)$, by sub-multiplicativity and $\|U_t^\top v\|^2 \leq \|U_t\|_F^2\|v\|^2 \leq \bar U^2 M$:
\[
|\Delta S_{ij}| \;\leq\; \frac{1}{2\tau}\left(\bar U^2 M + \bar U^2 M\right) = \frac{\bar U^2 M}{\tau}.
\]
Summing over $L^2$ entries gives $\|\Delta S\|_F^2 \leq L^2 (\bar U^2 M/\tau)^2$, hence $\|\Delta S\|_F \leq L\bar U^2 M/\tau$.
\end{proof}

\paragraph{Implications for the spectrum of $A^{\mathrm{Riem}}$.}
Since $S^{\mathrm{Riem}}$ already absorbs the $1/\tau$ scaling (Equation~\eqref{eq:riemscore}), the row-wise softmax Jacobian $\operatorname{diag}(p) - pp^\top$ is positive semidefinite with spectral norm $\leq \max_i p_i \leq 1$, so softmax is $1$-Lipschitz in Frobenius norm: $\|\softmax(S_1) - \softmax(S_2)\|_F \leq \|S_1 - S_2\|_F$ (each row is independent, and the row-wise $\ell_2$-Lipschitz constant is $\leq 1$). Combining this with $\|\cdot\|_{\mathrm{op}} \leq \|\cdot\|_F$ and Weyl's inequality, the singular values of $A^{\mathrm{Riem}}$ and $A^{\mathrm{Euc}}$ satisfy
\begin{equation}
\label{eq:weyl}
    \big|\sigma_k(A^{\mathrm{Riem}}) - \sigma_k(A^{\mathrm{Euc}})\big| \;\leq\; \|A^{\mathrm{Riem}} - A^{\mathrm{Euc}}\|_{\mathrm{op}} \;\leq\; \|A^{\mathrm{Riem}} - A^{\mathrm{Euc}}\|_F \;\leq\; \|\Delta S\|_F \;\leq\; \frac{L\, \bar U^2\, M}{\tau}.
\end{equation}
(We use the Frobenius-norm Lipschitz argument rather than a direct operator-norm bound because the row-wise softmax Jacobian does not decompose cleanly under the matrix operator norm; the Frobenius route is standard and suffices for the order-level conclusion.) This bound reveals a \emph{self-limiting structure}:
\begin{itemize}[leftmargin=*]
    \item \textbf{Sub-critical regime} ($\bar U^2 \ll \tau/L$, with $M = O(1)$): The spectra of $A^{\mathrm{Riem}}$ and $A^{\mathrm{Euc}}$ are close, so the near-degenerate spectrum that drives Euclidean collapse is preserved. The geometric correction is provably too weak to keep $\sigma_2$ bounded away from zero. Rank collapse proceeds essentially as in Theorem~\ref{thm:dong}.
    \item \textbf{Super-critical regime} ($\bar U^2 \gg \tau/L$): The bound becomes vacuous: $\|\Delta S\|_F$ can exceed $1$ (the scale of the spectrum of a row-stochastic matrix), so Weyl's inequality yields the trivial bound $|\sigma_k(A^{\mathrm{Riem}}) - \sigma_k(A^{\mathrm{Euc}})| \leq \|\Delta S\|_F = O(L\bar U^2/\tau)$, which for large $\bar U^2$ provides no information about the spectrum. Whether $\sigma_2(A^{\mathrm{Riem}})$ is bounded away from zero in this regime, i.e.\ whether collapse is actually prevented, is \emph{not addressed} by this bound.
\end{itemize}
In short, the perturbation bound tells us something only when the answer is negative (collapse still occurs). When the answer might be positive, the bound is silent. This is the central open problem of the framework.

\paragraph{What a rank-preservation proof would require.}
Conjecture~\ref{conj:rankpreserve} claims that for sufficiently diverse metrics, $\EffRank(H^{(n)}) > 1$ for all $n$. A direct proof would need to establish:
\begin{enumerate}[label=(\alph*)]
    \item The second-largest singular value $\sigma_2(A^{\mathrm{Riem}})$ is bounded away from zero by an amount depending on $\sigma_U^2$ and $r$;
    \item This bound propagates through depth: the product $\prod_{\ell} \sigma_2(A^{(\ell),\mathrm{Riem}})$ does not vanish exponentially.
\end{enumerate}
The non-Gram property (Theorem~\ref{thm:nongram}) addresses neither: it concerns the \emph{factorization} of $S$, not the \emph{spectrum} of $A$. The perturbation bound (Proposition~\ref{prop:perturb}) addresses (a) only in the regime $\bar U^2 = O(\tau/L)$, where the perturbation is too small to guarantee that $\sigma_2$ is bounded away from zero. Closing the gap to sufficiency requires controlling the spectrum of $A^{\mathrm{Riem}}$ in the \emph{large-perturbation} regime, which we now formalize.

\begin{conjecture}[Rank Preservation with Critical Regime and Directional Conditions]
\label{conj:rankstrong}
Let $D = \spn\{\mathbf{q}_i - \mathbf{k}_j : 1 \le i, j \le L\}$ denote the \emph{active query--key difference subspace} (Section~\ref{sec:counterexample}). A preliminary version of this conjecture claimed that sufficiently large metric strength $\bar U^2$ together with positive cross-token diversity $\sigma_U^2 > 0$ keeps $\sigma_2(A^{\mathrm{Riem}})$ bounded away from zero; that statement is \emph{false} as written, and Section~\ref{sec:counterexample} gives an explicit counterexample with $\bar U^2 \to \infty$ and $\sigma_U^2 > 0$ yet $\sigma_2(A^{\mathrm{Riem}}) = 0$. The corrected conjecture requires directional conditions:
\begin{itemize}[leftmargin=*]
    \item \textbf{Sub-critical regime} ($\bar U^2 < \bar U^2_{\mathrm{crit}}$ with $\bar U^2_{\mathrm{crit}} = \Theta(\tau/L)$): the Riemannian attention behaves as a perturbation of Euclidean attention (Proposition~\ref{prop:perturb}) and rank collapse proceeds essentially as in Theorem~\ref{thm:dong};
    \item \textbf{Super-critical regime} ($\bar U^2 > \bar U^2_{\mathrm{crit}}$, metric diversity $\sigma_U^2 > 0$, \emph{and} directional conditions: the metric factors overlap $D$, so that the correction $G_{ij}$ is not row-constant in any row, and the contrastive projected metric activity $G_{ij} - G_{ik}$ across competing keys is large relative to $\tau$): the attention rows are perturbed away from the Euclidean stationary distribution, $\sigma_2(A^{\mathrm{Riem}})$ is bounded away from zero by an amount depending on the contrastive projected metric activity rather than on $\bar U^2$ or $\sigma_U^2$ alone, and the effective rank of a pure Riemannian attention stack satisfies $\EffRank(H^{(n)}) \geq f(r, \sigma_U^2) > 1$ for all $n$. (Note: the anti-collapse condition is that $\sigma_2$ has a \emph{lower} bound; a large gap $\sigma_1 - \sigma_2$ would instead mean $\sigma_2 \ll \sigma_1$, i.e.\ proximity to rank-$1$, which is the collapsed regime.)
\end{itemize}
\end{conjecture}

\paragraph{Status of Conjecture~\ref{conj:rankstrong}.} The sub-critical regime follows directly from Proposition~\ref{prop:perturb} and the continuity of the spectrum in the score matrix. The naive super-critical claim, metric strength plus diversity alone, is \emph{disproved} by the orthogonal-subspace counterexample of Section~\ref{sec:counterexample}: the norms and diversity of $U_t$ do not control the overlap between the metric directions and the query--key differences that the attention scores actually compare, nor the contrast of the correction across keys. The corrected super-critical statement, with directional conditions, remains open: a proof would require controlling the spectrum of $A^{\mathrm{Riem}}$ through the contrastive quantities $G_{ij} - G_{ik}$, possibly via structured random-matrix analysis. Section~\ref{sec:counterexample} establishes a sufficient condition (permutation margins, Proposition~\ref{prop:margin}); establishing necessity, and closing the gap between the sub-critical and super-critical regimes, is Open Problem O1.

\paragraph{Summary of this subsection.} The non-Gram property (Theorem~\ref{thm:nongram}) is a structural characterization of Riemannian attention scores; it does not by itself address whether rank collapse occurs, because the core collapse mechanism (row-stochastic product convergence) does not depend on the Gram property. The perturbation analysis (Proposition~\ref{prop:perturb}) reveals a self-limiting structure: the bound is informative only in the sub-critical regime $\bar U^2 < \bar U^2_{\mathrm{crit}} = \Theta(\tau/L)$, where it confirms that collapse proceeds as in the Euclidean case; in the super-critical regime $\bar U^2 > \bar U^2_{\mathrm{crit}}$, where anti-collapse would need to occur, the bound becomes vacuous and standard perturbation tools provide no information. As the next subsection shows, the super-critical regime must additionally satisfy directional conditions, overlap of the metric factors with the active query--key difference subspace and contrastive (non-row-constant) corrections, because raw metric strength and diversity are insufficient. Closing the gap between the sub-critical regime (where collapse is proven) and the super-critical regime (where anti-collapse is conjectured under these conditions) is the central open problem identified by this framework (Open Problem O1).

\subsection{The Orthogonal-Subspace Counterexample and Necessary Conditions}
\label{sec:counterexample}

The perturbation bound of Proposition~\ref{prop:perturb} depends on $\bar U^2$, the maximum metric strength. We now show that \emph{no} bound in terms of $\bar U^2$ and the cross-token diversity $\sigma_U^2$ alone can certify anti-collapse: the metric correction can be exactly zero, with the attention matrix exactly that of Euclidean distance attention, while both quantities are arbitrarily large.

\paragraph{An explicit counterexample.} Take sequence length $L=2$, hidden dimension $d=4$, metric rank $r=1$, and any temperature $\tau > 0$. Let
\begin{equation}
    H^{(0)} = \begin{pmatrix} 1 & 0 & 0 & 0 \\ 0 & 1 & 0 & 0 \end{pmatrix}, \qquad
    W_Q = I_4, \qquad
    W_K = \begin{pmatrix} 1 & 0 & 0 & 0 \\ 2 & 1 & 0 & 0 \\ 0 & 0 & 1 & 0 \\ 0 & 0 & 0 & 1 \end{pmatrix},
\end{equation}
so that $q_1 = e_1$, $q_2 = e_2$, $k_1 = e_1$, $k_2 = 2e_1 + e_2$. Now choose genuinely different metric directions
\begin{equation}
    U_1 = M\, e_3, \qquad U_2 = 2M\, e_4, \qquad M > 0.
\end{equation}
Every query--key difference lies in $\spn\{e_1, e_2\}$: $q_1 - k_1 = 0$, $q_1 - k_2 = -e_1 - e_2$, $q_2 - k_1 = -e_1 + e_2$, $q_2 - k_2 = -2e_1$. Each difference is therefore orthogonal to both metric directions $e_3$ and $e_4$, so $U_i^\top(q_a - k_b) = 0$ for all $i, a, b$ and the metric correction $G_{ab} = \frac{1}{2}\|U_a^\top(q_a - k_b)\|^2 + \frac{1}{2}\|U_b^\top(q_a - k_b)\|^2$ vanishes identically. The squared-distance matrix is $E = \left[\begin{smallmatrix}0&2\\2&4\end{smallmatrix}\right]$, so $S^{\mathrm{Riem}} = -\frac{1}{\tau}\left[\begin{smallmatrix}0&2\\2&4\end{smallmatrix}\right]$; the second row is the first row minus the row-constant $2/\tau$, and row-wise softmax is invariant to row constants. Hence both attention rows are identical:
\begin{equation}
    A^{\mathrm{Riem}} = \begin{pmatrix} p & 1-p \\ p & 1-p \end{pmatrix}, \qquad p = \frac{1}{1 + e^{-2/\tau}},
\end{equation}
so $\rank(A^{\mathrm{Riem}}) = 1$ and $\sigma_2(A^{\mathrm{Riem}}) = 0$, and $H^{(1)} = A^{\mathrm{Riem}} H^{(0)}$ has two identical rows. Yet the metric quantities can be made arbitrarily large: $\bar U^2 = \max_t \|U_t\|_F^2 = 4M^2 \to \infty$ and $\sigma_U^2 = \Var_t[\|U_t\|_F^2] = \frac{9}{4}M^4 > 0$. For any proposed finite critical threshold scaling as $\Theta(\tau/L)$, $M$ can be chosen above it while $\sigma_2(A^{\mathrm{Riem}})$ remains exactly zero. This falsifies the super-critical implication of any conjecture stated purely in terms of $\bar U^2$ and $\sigma_U^2$ (the unqualified version of Conjecture~\ref{conj:rankstrong}).

\paragraph{The general obstruction.} Define the \emph{active query--key difference subspace}
\begin{equation}
\label{eq:activesubspace}
    D = \spn\{\mathbf{q}_i - \mathbf{k}_j : 1 \le i, j \le L\}.
\end{equation}
If every column of every $U_t$ lies in $D^\perp$, then $U_t^\top(\mathbf{q}_i - \mathbf{k}_j) = 0$ for all pairs, $G_{ij} = 0$, and the Riemannian attention matrix is exactly the Euclidean-distance attention matrix, while $\|U_t\|_F$, their variance across tokens, and their mutual directions can still be arbitrarily large and heterogeneous. Consequently, \emph{no} universal anti-collapse theorem can depend only on $\bar U^2$, $\sigma_U^2$, and $r$ without additionally controlling the overlap between the metric directions and $D$.

There is a second invariance, distinct from orthogonality. Even if $G_{ij} \neq 0$, a \emph{row-constant} metric correction has no effect: if $G_{ij} = c_i$ for all $j$, then $\softmax(S^{\mathrm{Euc}}_i - c_i/\tau) = \softmax(S^{\mathrm{Euc}}_i)$. The operational quantity is therefore not raw or projected metric energy, but \emph{contrast across competing keys}, e.g.\ $G_{ij} - G_{ik}$ or a row-centered version of $G$.

\paragraph{A sufficient condition: contrastive permutation margins.}
\label{sec:margin}
The preceding obstructions indicate what a correct positive statement must measure: how much the metric correction separates the intended key from its competitors. The following sufficient condition formalizes this.

\begin{proposition}[Permutation-Margin Spectral Preservation]
\label{prop:margin}
Let $A^{(\ell)} = \softmax(S^{(\ell)})$ be a row-stochastic attention matrix at layer $\ell$. Suppose that for each layer there is a permutation $\pi_\ell$ of the $L$ token positions such that
\begin{equation}
    S^{(\ell)}[i, \pi_\ell(i)] - S^{(\ell)}[i, j] \;\geq\; \gamma_\ell \qquad \text{for all } j \neq \pi_\ell(i),
\end{equation}
and define $\varepsilon_\ell = \sqrt{2L}(L-1) e^{-\gamma_\ell}$. If $\varepsilon_\ell < 1$ and $\sum_\ell \varepsilon_\ell < \infty$, then for a pure stack $H^{(n)} = A^{(n-1)} \cdots A^{(0)} H^{(0)}$,
\begin{equation}
    \inf_n \frac{\sigma_2(H^{(n)})}{\sigma_1(H^{(n)})} \;>\; 0,
\end{equation}
provided $\sigma_2(H^{(0)}) > 0$; the relative singular spectrum cannot collapse to rank one.
\end{proposition}

\begin{proof}
Let $P^{(\ell)}$ be the permutation matrix of $\pi_\ell$. By softmax, $1 - A^{(\ell)}[i, \pi_\ell(i)] \le (L-1)e^{-\gamma_\ell}$ for each row $i$, so each row of $A^{(\ell)}$ is within Euclidean distance $\sqrt{2}(L-1)e^{-\gamma_\ell}$ of the corresponding row of $P^{(\ell)}$, giving $\|A^{(\ell)} - P^{(\ell)}\|_2 \le \|A^{(\ell)} - P^{(\ell)}\|_F \le \sqrt{2L}(L-1)e^{-\gamma_\ell} = \varepsilon_\ell$. All singular values of $P^{(\ell)}$ equal $1$, so singular-value perturbation yields $\sigma_{\min}(A^{(\ell)}) \ge 1 - \varepsilon_\ell$ and $\sigma_{\max}(A^{(\ell)}) \le 1 + \varepsilon_\ell$. For a product of such matrices, $\sigma_k(H^{(n)}) \ge \sigma_k(H^{(0)}) \prod_{\ell<n}(1-\varepsilon_\ell)$ and $\sigma_1(H^{(n)}) \le \sigma_1(H^{(0)}) \prod_{\ell<n}(1+\varepsilon_\ell)$. If $\varepsilon_\ell < 1$ and $\sum_\ell \varepsilon_\ell < \infty$, the product $\prod_\ell \frac{1-\varepsilon_\ell}{1+\varepsilon_\ell}$ is strictly positive, so $\inf_n \sigma_2(H^{(n)})/\sigma_1(H^{(n)}) > 0$ whenever $\sigma_2(H^{(0)}) > 0$.
\end{proof}

\paragraph{Riemannian specialization.} Since $S^{\mathrm{Riem}}_{ij} = -(E_{ij} + G_{ij})/\tau$, the margin condition follows if, for every competitor $j \neq \pi(i)$,
\begin{equation}
\label{eq:contrastivemargin}
    (E_{ij} - E_{i\pi(i)}) + (G_{ij} - G_{i\pi(i)}) \;\geq\; \tau\, \gamma_\ell.
\end{equation}
This identifies the missing geometric variable precisely: \emph{contrastive projected metric activity}, the amount by which the metric correction favors the designated key over its competitors, rather than raw metric magnitude. Metrics that are strong and diverse but do not create such contrast (e.g., by acting on $D^\perp$, or by adding a row-constant correction) cannot contribute to anti-collapse, exactly as the counterexample demonstrates.

\subsection{Expressivity: Parametric Capacity Analysis}

\begin{proposition}[Metric Parametric Capacity]
\label{prop:expressivity}
A token representation $h_t \in \mathbb{R}^d$ can modulate its pairwise distances to other tokens in two ways:
\begin{enumerate}[label=(\alph*)]
    \item \textbf{Vector movement}: Changing $h_t$ itself provides $d$ real-valued coordinates.
    \item \textbf{Metric modulation}: A rank-$r$ Riemannian metric $g_t = I + U_t U_t^\top$ with $U_t = \mathrm{MetricNet}(h_t) \in \mathbb{R}^{d \times r}$ is parameterized by a $d \times r$ matrix. The metric generator (MetricNet) has $O(d^2 \cdot r)$ trainable parameters (a two-layer MLP), but the per-token metric factor $U_t$ is a deterministic function of the $d$-dimensional input $h_t$, so the information bottleneck remains $d$-dimensional. The additional \emph{parametric} capacity (the MLP weights) allows the model to learn a mapping from token representations to metric geometries, and the $d \cdot r$ output dimensions of $U_t$ provide a richer \emph{representation} of the token's geometric environment than the $d$-dimensional hidden state alone.
\end{enumerate}
The metric generator provides $O(d^2 r)$ learnable parameters beyond the standard Transformer, but the information-theoretic degrees of freedom available to a single token are still bounded by the input dimension $d$. The benefit is architectural: MetricNet learns to allocate geometric capacity (stretching specific directions for specific token types) in a way that a flat Euclidean space cannot express, even though each $U_t$ is a deterministic function of $h_t$.
\end{proposition}

\begin{proof}
The metric $g_t = I + U_t U_t^\top$ is a $d \times d$ symmetric matrix. The space of $d \times d$ symmetric matrices has dimension $\frac{d(d+1)}{2}$, and the positive-definite cone has the same dimension. The map $U \mapsto I + UU^\top$ from $\mathbb{R}^{d \times r}$ to $\text{SPD}(d)$ has image equal to the manifold of rank-$\leq r$ positive-semidefinite perturbations of $I_d$, which has dimension $dr - r(r-1)/2$ (the dimension of the rank-$r$ PSD manifold); this is $\Theta(dr)$ for $r \leq d$. So the low-rank parameterization can explore a $dr$-dimensional submanifold of $\mathrm{SPD}(d)$. However, since $U_t = \mathrm{MetricNet}(h_t)$, the actual degrees of freedom per token are constrained by the $d$-dimensional input. The architectural value lies in MetricNet's ability to map different regions of the $d$-dimensional hidden space to different metric geometries, a nonlinear mapping that a purely vector-based architecture cannot replicate.
\end{proof}

\begin{remark}
The degrees-of-freedom advantage is not merely quantitative; it is \emph{qualitative}. Vector movement changes \emph{all} pairwise distances involving $h_t$ simultaneously (because moving $h_t$ changes $\|h_t - h_s\|$ for every $s$). Metric modulation can change the distance to token $s$ \emph{independently} of the distance to token $s'$, by stretching more along the direction of $h_t - h_s$ while stretching less along $h_t - h_s'$. (Note: with $g_t = I + U_t U_t^\top \succeq I_d$, the metric can only \emph{stretch} directions; all eigenvalues are $\geq 1$, so it never compresses them below the Euclidean baseline; see the limitation discussed in Section~\ref{sec:naturalffn}. The independent modulation comes from stretching \emph{different} directions by \emph{different} amounts, not from compression.) This independent modulation is impossible in Euclidean geometry, where all distances are coupled through the single vector $h_t$.
\end{remark}

\subsection{The Nash Embedding Objection}

A natural objection: the Nash Embedding Theorem \cite{nash1956imbedding} proves that any Riemannian manifold can be isometrically embedded in a sufficiently high-dimensional Euclidean space. Could a wide enough Euclidean hidden space subsume any non-Euclidean geometry, rendering explicit metrics unnecessary?

He et al.\ \cite{he2025position} provide a systematic refutation relevant to learned representations:

\begin{enumerate}[label=(\arabic*)]
    \item \textbf{Dimensional blowup.} The smooth Nash embedding guarantees an isometric embedding into $\mathbb{R}^n$ with $n \leq \frac{m(3m+11)}{2}$ for an $m$-dimensional manifold. For $m \sim 10^2$ (typical per-head dimensions), this already requires $n \sim 10^4$, exceeding practical Transformer dimensions and defeating the purpose.
    \item \textbf{Insufficient differentiability.} The $C^1$ Nash-Kuiper embedding achieves lower target dimensions but is only once differentiable ($C^1$, not $C^2$), so the curvature tensor is ill-defined and the embedding is too irregular for gradient-based learning that relies on second-order geometric structure.
    \item \textbf{Non-constructivity.} The theorem is existential: there is no known training algorithm that causes gradient descent to discover a Nash-style isometric embedding. The Euclidean geometry of Transformer hidden spaces is a genuine structural constraint, not one surmountable by adding dimensions.
\end{enumerate}

The conclusion is that explicit Riemannian metrics provide a form of geometric expressivity that cannot be replicated by simply widening a Euclidean hidden space, provided the metric actually interacts with the directions in which queries and keys differ. Section~\ref{sec:counterexample} shows that metrics acting only on the orthogonal complement of the active query--key difference subspace $D$ have no effect on attention at all; the expressivity argument is therefore conditional on the metric directions overlapping $D$.

\subsection{Heterogeneous vs.\ Constant-Curvature Geometry}

Prior work on non-Euclidean representations has focused on \emph{constant-curvature} spaces, predominantly hyperbolic spaces with global negative curvature \cite{nickel2017poincare,he2025helm}. However, natural language semantics are not uniformly curved. Function words (``the,'' ``is,'' ``of'') inhabit nearly flat regions; content words span richer geometries; specialized terminology activates highly curved, context-specific structures.

The semantic manifold is \emph{heterogeneous}: curvature varies continuously by token and context. This motivates our use of \emph{learned, per-token} Riemannian metrics rather than a single global curvature parameter. The low-rank representation $g_t = I + U_t U_t^\top$ allows the curvature at token $t$, measured by $\kappa_t = \|U_t\|_F^2$, to vary freely across positions, from near-zero (flat, Euclidean-like) to large values (strongly curved). The key distinction from constant-curvature approaches \cite{gulcehre2019hyperbolic,yang2024hypformer}, which apply a single global curvature to all tokens, is this per-position variability; we note that learned per-token geometry has also been explored in recent geometric attention frameworks \cite{ji2025riemannformer,lin2025cat,li2026mahalanobis}, and the present analysis additionally characterizes when such per-token geometry can and cannot affect attention (Section~\ref{sec:counterexample}).

\section{Efficient Computation via Low-Rank Factorization}
\label{sec:efficient}

A full per-token metric $g_t \in \mathbb{R}^{d \times d}$ would require $O(d^2)$ storage and $O(d^3)$ operations per token, which is infeasible at scale. The low-rank representation $g_t = I + U_t U_t^\top$ renders the metric operations tractable; the dominant residual cost is metric \emph{generation} itself, since MetricNet (Section~\ref{sec:metricnet}) has $O(d^2 r)$ parameters and costs $O(BL d^2 r)$ per layer (Proposition~\ref{prop:complexity}).

\subsection{Metric Inversion via the Woodbury Identity}

The metric-preconditioned update in the feed-forward sublayer (Section~\ref{sec:naturalffn}) requires $g_t^{-1}$. The Woodbury identity yields a closed form:

\begin{proposition}[Efficient Metric Inversion]
\label{prop:woodbury}
For $g_t = I_d + U_t U_t^\top$ with $U_t \in \mathbb{R}^{d \times r}$:
\begin{equation}
\label{eq:woodbury}
    g_t^{-1} = I_d - U_t (I_r + U_t^\top U_t)^{-1} U_t^\top
\end{equation}
The inversion reduces from $O(d^3)$ (for a general $d \times d$ matrix) to $O(d \cdot r^2 + r^3)$, which is $O(d \cdot r^2)$ when $r \ll d$. For $d = 4096, r = 8$: $\sim 2.6 \times 10^5$ vs.\ $\sim 6.9 \times 10^{10}$, a $\sim 10^5\times$ speedup.
\end{proposition}

\begin{proof}
The Woodbury identity states $(A + CBC^\top)^{-1} = A^{-1} - A^{-1}C(B^{-1} + C^\top A^{-1}C)^{-1}C^\top A^{-1}$. Setting $A = I_d$, $B = I_r$, $C = U_t$ gives the result. The inner matrix $I_r + U_t^\top U_t \in \mathbb{R}^{r \times r}$ is inverted in $O(r^3)$; the matrix products are $O(d \cdot r^2)$.
\end{proof}

\subsection{Geodesic Distance Decomposition}

The geodesic distance under the pairwise metric $g_{ij}$ decomposes into per-token quantities:

\begin{proposition}[Geodesic Distance Decomposition]
\label{prop:geodesic}
The squared geodesic distance $d_g^2(\mathbf{q}_i, \mathbf{k}_j)$ under $g_{ij} = I + \frac{U_i U_i^\top + U_j U_j^\top}{2}$ decomposes as:
\begin{equation}
\label{eq:geodecomp}
    d_g^2 = \underbrace{\|\mathbf{q}_i - \mathbf{k}_j\|^2}_{\text{Euclidean}} + \underbrace{\frac{1}{2}\|U_i^\top(\mathbf{q}_i - \mathbf{k}_j)\|^2}_{\text{metric } i} + \underbrace{\frac{1}{2}\|U_j^\top(\mathbf{q}_i - \mathbf{k}_j)\|^2}_{\text{metric } j}
\end{equation}
Each metric term further decomposes:
\begin{equation}
\label{eq:metricdecomp}
    \|U_i^\top(\mathbf{q}_i - \mathbf{k}_j)\|^2 = \|U_i^\top \mathbf{q}_i\|^2 + \|U_i^\top \mathbf{k}_j\|^2 - 2\,(U_i^\top \mathbf{q}_i)^\top (U_i^\top \mathbf{k}_j)
\end{equation}
The per-token quantities $\|U_i^\top \mathbf{q}_i\|^2$ and $U_i^\top \mathbf{q}_i$ are precomputed once per token in $O(d \cdot r)$. The pairwise term $(U_i^\top \mathbf{q}_i)^\top (U_i^\top \mathbf{k}_j)$ requires $O(r)$ per pair.
\end{proposition}

\subsection{Complexity Analysis}

\begin{proposition}[Computational Complexity of Riemannian Attention]
\label{prop:complexity}
For a sequence of length $L$, hidden dimension $d$, batch size $B$, and metric rank $r$:
\begin{center}
\begin{tabular}{lcc}
\toprule
\textbf{Operation} & \textbf{Standard Attention} & \textbf{Riemannian Attention} \\
\midrule
Score computation & $O(BL^2 d)$ & $O(BL^2 d) + O(BL^2 r)$ \\
Metric generation & --- & $O(BL d^2 r)$ (MLP, $m=1$) \\
Metric inversion (per FFN) & --- & $O(BLdr^2)$ \\
\bottomrule
\end{tabular}
\end{center}
The score-level geometric overhead is $O(r/d)$: the metric corrections add $O(r)$ per pair to the $O(d)$ Euclidean score. MetricNet (Section~\ref{sec:metricnet}) is a two-layer MLP with $O(d^2 r)$ parameters (Proposition~\ref{prop:expressivity}) and per-layer cost $O(BL d^2 r)$, i.e.\ an $O(dr/L)$ overhead relative to the $O(BL^2 d)$ attention cost; the total overhead ratio is therefore $O(r/d + dr/L)$, which is negligible when $r \ll d$ \emph{and} $L \gg dr$. A bottlenecked MetricNet variant (hidden width $O(r)$ instead of $dm$) reduces the generation cost to $O(BL d r^2)$, an $O(r^2/L)$ overhead. The full $d \times d$ metric matrix $g_t$ is \emph{never materialized} in any operation.
\end{proposition}

\begin{proof}
Standard attention: $QK^\top$ costs $O(BL^2 d)$ (per head, summed over heads gives the same). Riemannian attention: the Euclidean part $\|\mathbf{q}_i - \mathbf{k}_j\|^2$ costs $O(BL^2 d)$; the metric corrections require $O(BLdr)$ for per-token precomputation ($U_i^\top \mathbf{q}_i$, $U_i^\top \mathbf{k}_j$, etc.) and $O(BL^2 r)$ for the pairwise interaction, i.e.\ a score-level overhead of $O(r/d)$ per pair. Metric generation: MetricNet is the two-layer MLP of Section~\ref{sec:metricnet} with weights $W_1 \in \mathbb{R}^{d \times dm}$ and $W_2 \in \mathbb{R}^{dm \times dr}$, so the forward pass costs $O(d \cdot dm + dm \cdot dr) = O(d^2 m r)$ per token, i.e.\ $O(BL d^2 r)$ per layer for the default $m = 1$, an $O(dr/L)$ overhead relative to attention; a bottlenecked variant with hidden width $O(r)$ costs $O(d r^2)$ per token instead. Metric inversion: by Proposition~\ref{prop:woodbury}, $O(d r^2)$ per token, so $O(BLdr^2)$ total. The total additional cost over standard attention is $O(BL^2 r + BL d^2 r + BL d r^2)$, which is small relative to $O(BL^2 d)$ when $r \ll d$ and $L \gg dr$.
\end{proof}

\begin{remark}[Connection to FlashAttention]
The decomposition in Proposition~\ref{prop:geodesic} is compatible in principle with fused attention kernels (e.g., FlashAttention \cite{dao2022flashattention}): the metric correction can be computed as an additive bias to the attention scores, in the same spirit as ALiBi \cite{press2022alibi} or relative position embeddings, and fused into the softmax without materializing the full $L \times L$ score matrix in HBM. This indicates that Riemannian attention can be integrated within existing efficient attention kernels; an implementation-level demonstration, with measured throughput, is future work.
\end{remark}

\section{Architecture Design: The Fiber Bundle Transformer}
\label{sec:architecture}

We now present the \emph{Fiber Bundle Transformer}, a complete architecture specification grounded in the theory of Sections~\ref{sec:theory}--\ref{sec:efficient}. The geometric framing, attention as metric-dependent similarity in a curved space with the metric induced by the token representations and the connection implementing parallel transport, has been explored in prior work \cite{ji2025riemannformer,disipio2025curved}; our contribution is not the framing itself but the specific low-rank parameterization $g_t = I + U_t U_t^\top$, the theoretical analysis of Sections~\ref{sec:theory}--\ref{sec:efficient} (including the necessary and sufficient conditions of Section~\ref{sec:counterexample}), and the identification of the failure modes in Section~\ref{sec:predictions}. The design is presented as a \emph{theoretical proposal} with mathematical justification; we do not claim empirical validation, and we do not claim that the geometric framing is novel.

\subsection{Fiber Bundle Formulation}
\label{sec:fiberbundle}

\begin{definition}[Fiber bundle attention model]
Consider the sequence of token positions as a discrete base space $\mathcal{B} = \{1, 2, \ldots, L\}$. At each position $t$, attach a \emph{fiber} $\mathcal{F}_t$, which is a copy of the hidden space $\mathbb{R}^d$ equipped with its own Riemannian metric $g_t = I + U_t U_t^\top$.

The hidden state $h_t$ is a \emph{section} of this fiber bundle: a choice of one vector from each fiber. In the standard Transformer, all fibers share the Euclidean metric, so a vector $v \in \mathbb{R}^d$ has the same meaning in every fiber. In the Fiber Bundle Transformer, the same vector $v$ has \emph{different semantic interpretations} in different fibers because inner products are metric-dependent:
\begin{equation}
    \langle v, w \rangle_{\text{at fiber } t} = v^\top g_t \, w \;\neq\; \langle v, w \rangle_{\text{at fiber } s} = v^\top g_s \, w
\end{equation}
\end{definition}

The \emph{connection} on the bundle, the rule for comparing vectors across fibers, is given by a parallel transport operator $P_{s \to t}$ that approximately preserves the metric:
\begin{equation}
    P_{s \to t}^\top \, g_t \, P_{s \to t} \;\approx\; g_s
\end{equation}
The deviation from exact preservation measures the \emph{curvature} of the connection; the antisymmetric part of the connection measures \emph{torsion}.

\paragraph{On the terminology.} The name ``Fiber Bundle Transformer'' is a \emph{structural analogy} to the differential-geometric notion of a fiber bundle, not a strict construction in the sense of, e.g., Kobayashi--Nomizu. In particular, we do not specify a structure group, transition functions between local trivializations, or a smooth atlas on the base space; the discrete base $\mathcal{B} = \{1,\ldots,L\}$ does not carry a smooth structure. What we retain from the differential-geometric picture is the structural intuition, independent fibers carrying different per-position metrics with a connection defining cross-fiber transport, and this intuition is accurately captured. The curvature and torsion objects we introduce below (Phases B, C) are correspondingly \emph{proxies} motivated by their differential-geometric counterparts, not the Riemann and torsion tensors of a true connection on a smooth bundle. A rigorous differential-geometric formulation of the construction (specifying a structure group, e.g.\ $\mathrm{GL}(d)$, and a discrete connection compatible with the per-token metrics) is left as open work (Open Problem O6).

Each Fiber Bundle Transformer layer consists of five core phases (A--E), with an optional enhanced projection (Phase A++), operating on the hidden state $x \in \mathbb{R}^{B \times L \times d}$ as the persistent residual stream, with geometric data $(U_t, P_t, T_t)$ generated ephemerally at each layer from $x$.

\subsection{Phase A: Metric Generation (MetricNet)}
\label{sec:metricnet}

\begin{definition}[MetricNet]
MetricNet is a two-layer MLP that maps each token's hidden state to a low-rank metric factor:
\begin{equation}
    U_t = \mathrm{MetricNet}(h_t) \in \mathbb{R}^{d \times r}, \qquad g_t = I_d + U_t U_t^\top
\end{equation}
The MLP architecture is $\text{Linear}(d, d \cdot m) \to \mathrm{SiLU} \to \text{Linear}(d \cdot m, d \cdot r)$, where $m$ is an internal expansion ratio (typically $m = 1$) and $r$ is the metric rank (typically $r \in \{4, 8\}$).
\end{definition}

\paragraph{Design rationale.} MetricNet generates the per-token metric $g_t$ from the token's own hidden state $h_t$, making the geometry \emph{content-dependent}. A token in a financial context generates a metric that stretches the ``money'' direction; the same token in a geographical context generates a metric that stretches the ``river'' direction. This is the mechanism by which the model exploits the $O(d \cdot r)$ degrees of freedom identified in Proposition~\ref{prop:expressivity}.

\paragraph{Initialization.} The final layer of MetricNet should be initialized with small weights (e.g., $\mathcal{N}(0, \sigma^2/d)$ with moderate $\sigma$), so that $U_t \approx 0$ at initialization and $g_t \approx I_d$ (Euclidean). This ensures the architecture starts in the Euclidean regime (distance-based attention with $g \approx I$, close to a standard Transformer under key-norm normalization; see Section~\ref{sec:blending}) and gradually learns non-Euclidean geometry, avoiding training instability from large initial metric perturbations.

\subsection{Phase A++: Metric-Aware Projection}
\label{sec:metricproj}

In the enhanced variant, Q/K/V projections incorporate metric information directly, making the query/key representations live in a curved coordinate system:

\begin{equation}
    Q_i^{\mathrm{geo}} = W_Q h_i + \gamma \cdot \frac{1}{\sqrt{d}} \, U_i U_i^\top (W_Q h_i)
\end{equation}
where $\gamma = \sigma(\gamma_{\mathrm{param}})$ is a learned gating parameter initialized small ($\gamma_{\mathrm{param}} = -3 \Rightarrow \gamma \approx 0.047$). The second term adds the component of $W_Q h_i$ lying in the column space of $U_i$, the metric's preferred directions. This is motivated by the fact that the squared geodesic distance under $g_{ij}$ takes the form $d_g^2 = (Q_i - K_j)^\top g_{ij} (Q_i - K_j) = \|g_{ij}^{1/2}(Q_i - K_j)\|^2$, and the metric-aware projection approximates the $g_{ij}^{1/2}$ correction (expanded to first order as $g_{ij}^{1/2} \approx I + \tfrac{1}{2} U_i U_i^\top$ for small $U_i$) without computing a matrix square root.

\subsection{Phase B: Curvature Proxy}
\label{sec:curvature}

The curvature of a true fiber-bundle connection measures the path-dependence of parallel transport. We do not have a true connection here (see the caveat in Section~\ref{sec:fiberbundle}); instead, we define a per-token \emph{curvature proxy} that captures the same intuition, that tokens whose metric deviates strongly from the Euclidean baseline inhabit a more curved region of the representational space:

\begin{equation}
    \Omega_t = \softplus(\beta) \cdot \|U_t\|_F^2
\end{equation}
where $\beta$ is a learned scale parameter.

\paragraph{On the proxy.} We are explicit that $\Omega_t$ is a \emph{heuristic scalar} and \emph{not} the Riemannian sectional curvature of the metric $g_t = I + U_t U_t^\top$. The sectional curvature of such a rank-$r$ perturbation of the Euclidean metric is a tensor-valued quantity depending on $U_t$ and its derivatives with respect to the base coordinate; it does not reduce to a scalar function of $\|U_t\|_F^2$ alone. The proxy is motivated by the qualitative observation that, for metrics of the form $g = I + UU^\top$, the magnitude of the metric perturbation (hence the deviation from flatness) is controlled by $\|U\|_F^2$. Since $\|U_t\|_F^2 \geq 0$ already and $\softplus(\beta) > 0$, the proxy $\Omega_t$ is non-negative (matching the intuition that curvature magnitude is non-negative); the $\softplus$ on $\beta$ ensures a strictly positive learned scale factor. The dependence on $\|U_t\|_F^2$ means tokens with large metric perturbation receive a stronger curvature penalty in attention (Phase D). A scalar proxy of this form is sufficient for the architectural role we ask it to play (a per-token attention bias in Phase D); replacing it with a tensorial curvature object is left as a refinement (Open Problem O6).

\subsection{Phase C: Torsion Generation (TorsionNet)}
\label{sec:torsion}

\begin{definition}[TorsionNet]
TorsionNet generates low-rank factors $L_t, R_t \in \mathbb{R}^{d \times r_T}$ from the hidden state via a separate MLP. The per-token torsion tensor is:
\begin{equation}
    T_t = L_t R_t^\top - R_t L_t^\top \in \mathfrak{so}(d), \qquad T_t^\top = -T_t
\end{equation}
The antisymmetry $T_t^\top = -T_t$ is structural: it ensures that the torsion-induced correction to transport is \emph{directional}, encoding $A \to B \neq B \to A$.
\end{definition}

\paragraph{Design rationale.} In standard Transformers, attention is symmetric in the sense that the similarity of $i$ to $j$ and $j$ to $i$ are computed by the same dot product (modulo softmax normalization). Torsion breaks this symmetry: the transport of information from $j$ to $i$ can differ from $i$ to $j$, reflecting that semantic relationships are inherently directional (``$A$ causes $B$'' is not the same as ``$B$ causes $A$''). The antisymmetric construction $T = LR^\top - RL^\top$ is the canonical way to generate an element of $\mathfrak{so}(d)$ from low-rank factors, ensuring the torsion never has a symmetric component.

\paragraph{On the roles of $\Delta_i$ and $T_i$.} Both the connection correction $\Delta_i = A_i B_i^\top - B_i A_i^\top$ (Phase D) and the torsion $T_i = L_t R_t^\top - R_t L_t^\top$ (Phase C) are antisymmetric matrices in $\mathfrak{so}(d)$, and from a purely algebraic standpoint they enter the transport $P_i$ in the same way. Their distinction is \emph{functional}, not algebraic: $\Delta_i$ is the default transport correction applied to every token, while $T_i$ is an \emph{additional} directionality-sensitive correction gated by $\gamma_T$ (recommended initialized small). In a true differential-geometric connection, torsion is the antisymmetric part of the connection and is structurally distinct from the connection itself; here, where both are heuristic proxies (Section~\ref{sec:fiberbundle}), the two modules serve as independent learned channels for the transport correction, one content-general ($\Delta_i$) and one directionality-specialized ($T_i$). Collapsing them into a single module is a valid simplification; we retain both to make the geometric analogy explicit and to allow the directionality signal to be isolated, probed, and ablated independently.

\subsection{Phase D: Geodesic Attention}
\label{sec:geodesicattn}

Attention computes geodesic distances with curvature modulation:

\begin{equation}
\label{eq:geodesicattn}
    \alpha_{ij} = \softmax_j\!\left(-\frac{d_g^2(\mathbf{q}_i, \mathbf{k}_j) + \Omega_i + \Omega_j}{\tau}\right)
\end{equation}
where $d_g^2$ is given by Equation~\eqref{eq:geodecomp} and $\Omega_t$ is the curvature proxy from Phase B. The temperature $\tau = \sqrt{d/h}$ matches standard scaled dot-product attention.

The curvature terms $\Omega_i + \Omega_j$ act as an \emph{attention penalty}: tokens in high-curvature regions (semantically complex) require higher raw similarity to attend to each other. This implements the geometric intuition that curved regions of the semantic manifold are ``harder to traverse''; information transport across them should be more selective.

After aggregation, values are transported via the connection:
\begin{equation}
    \mathrm{output}_i = P_i \cdot \left(\sum_j \alpha_{ij} \mathbf{v}_j\right)
\end{equation}
The transport operator $P_i = I_{d} + \Delta_i$ is generated by a FiberTransport module, where $\Delta_i = A_i B_i^\top - B_i A_i^\top$ is antisymmetric. When torsion is active, the transport becomes $P_i' = P_i + \gamma_T \cdot T_i$.

\paragraph{Isometry analysis.} The antisymmetry $\Delta_i^\top = -\Delta_i$ ensures that $P_i$ is \emph{approximately} orthogonal, but not exactly:
\begin{equation}
\label{eq:isometryerr}
    P_i^\top P_i = (I + \Delta_i^\top)(I + \Delta_i) = I + \underbrace{\Delta_i + \Delta_i^\top}_{=\,0} + \Delta_i^\top \Delta_i = I + \Delta_i^\top \Delta_i.
\end{equation}
The error matrix $\Delta_i^\top \Delta_i \succeq 0$ has rank at most $2r_T$ (the rank of $\Delta_i$) and Frobenius norm $\|\Delta_i^\top \Delta_i\|_F = \|\Delta_i\|_F^2$. Hence $P_i$ is an isometry up to an additive positive-semidefinite error of spectral norm $\|\Delta_i^\top \Delta_i\|_{\mathrm{op}} \leq \|\Delta_i\|_{\mathrm{op}}^2 = O(\|\Delta_i\|_F^2)$. For $\|\Delta_i\|_F \ll 1$, $P_i$ is approximately orthogonal with second-order accuracy; for $\|\Delta_i\|_F = O(1)$, the deviation from isometry is $O(1)$ and the transport cannot be regarded as metric-preserving. A fully isometric transport can be recovered by replacing $P_i = I + \Delta_i$ with the Cayley transform $P_i = (I - \Delta_i/2)^{-1}(I + \Delta_i/2)$, which maps antisymmetric matrices into $\mathrm{SO}(d)$ exactly at the cost of an $O(d r_T^2)$ matrix solve; we adopt the linear form for simplicity and note the Cayley alternative as a refinement. We note that (near-)orthogonal constraints on attention parameters \cite{fei2022ovit} and (near-)invertible attention constructions \cite{kim2021lipschitz,zha2021invertible} are established techniques; the analysis above applies the same toolkit to the transport operator. The torsion correction $P_i' = P_i + \gamma_T T_i$ introduces an additional isometry violation of order $O(\gamma_T^2 \|T_i\|_F^2)$, which is small when $\gamma_T$ is initialized small (as we recommend).

\subsection{Phase E: Metric-Preconditioned FFN}
\label{sec:naturalffn}

The feed-forward sublayer replaces the Euclidean update step with a \emph{metric-preconditioned} step inspired by the Riemannian-gradient construction:

\begin{equation}
\label{eq:naturalgrad}
    h_t \leftarrow h_t + g_t^{-1} \cdot \mathrm{FFN}(h_t)
\end{equation}
where $g_t^{-1}$ is computed via the Woodbury identity (Proposition~\ref{prop:woodbury}) in $O(d \cdot r^2)$.

\paragraph{Terminology and rationale.} Amari's \emph{natural gradient} \cite{amari1998natural} is defined in the \emph{parameter space} of a statistical model, with the metric given by the Fisher information. The update in Equation~\eqref{eq:naturalgrad} is not a natural gradient in this strict sense: it operates in the \emph{representation space} $\mathbb{R}^d$, with the metric given by the learned $g_t$ rather than a Fisher matrix. We therefore refer to it as a \emph{metric-preconditioned} update, or equivalently a \emph{Riemannian-gradient-inspired} update, following the same geometric intuition: on a Riemannian manifold with metric $g$, the direction of steepest ascent of a function $f$ is $g^{-1}\nabla f$, not $\nabla f$.

By applying $g_t^{-1}$ to the FFN output, the update respects the local geometry. We emphasize that with the chosen parameterization $g_t = I_d + U_t U_t^\top \succeq I_d$, the metric only \emph{stretches} directions (all eigenvalues are $\geq 1$); there are no ``compressed'' directions, so the general Riemannian-gradient intuition (that $g^{-1}$ amplifies directions with small eigenvalues of $g$) does not directly apply here. Concretely, directions in the column space of $U_t$ (eigenvalues $1 + \sigma_i^2 > 1$, stretched by $g_t$) are correspondingly \emph{shrunk} by $g_t^{-1} \preceq I_d$ and receive \emph{smaller} updates, while directions in the null space of $U_t^\top$ (eigenvalue $1$, unstretched) receive unchanged updates. The net effect is to down-weight FFN updates along the metric's preferred (stretched) directions. This is the geometric analogue of adaptive learning rates (cf.\ Adam's per-parameter scaling), but operating in the \emph{representation space} rather than the \emph{parameter space}. Whether this representation-space preconditioning accelerates training convergence in practice, relative to a Euclidean FFN, is an open empirical question (Open Problem O4).

\subsection{Hybrid Inference Blending}
\label{sec:blending}

For practical deployment, the geometric correction can be blended with standard Euclidean attention:

\begin{equation}
\label{eq:blending}
    s_{ij}^{\mathrm{hybrid}} = s_{ij}^{\mathrm{Euc}} + \alpha \cdot \Delta s_{ij}^{\mathrm{geo}}
\end{equation}
where $s_{ij}^{\mathrm{Euc}} = \mathbf{q}_i^\top \mathbf{k}_j / \sqrt{d}$ is the standard score, $\Delta s_{ij}^{\mathrm{geo}} = s_{ij}^{\mathrm{Riem}} - s_{ij}^{\mathrm{Euc}}$ is the metric correction, and $\alpha \in [0, 1]$ is a blending coefficient. At $\alpha = 0$, the model reduces to a standard Transformer; at $\alpha = 1$, it is fully Riemannian. This enables a smooth interpolation between Euclidean and Riemannian attention, which is critical for the training/inference consistency analysis in Section~\ref{sec:predictions}.

\paragraph{On the two score parameterizations.} The Riemannian score (Equation~\eqref{eq:riemscore}) is \emph{distance-based}, $s_{ij}^{\mathrm{Riem}} = -d_g^2(\mathbf{q}_i, \mathbf{k}_j)/\tau$, whereas the Euclidean baseline in the blending formula is \emph{dot-product-based}, $s_{ij}^{\mathrm{Euc}} = \mathbf{q}_i^\top \mathbf{k}_j / \sqrt{d}$. These two parameterizations are \emph{not directly comparable}: the distance form contains terms $-\|\mathbf{q}_i\|^2/\tau$ and $-\|\mathbf{k}_j\|^2/\tau$ that have no counterpart in the dot-product form. Specifically, with $\tau = \sqrt{d/h}$ (matching standard per-head scaling),
\begin{equation}
    s_{ij}^{\mathrm{Riem}} - s_{ij}^{\mathrm{Euc}} = \underbrace{-\frac{\|\mathbf{q}_i\|^2}{\tau}}_{\text{constant in row } i} \;-\; \underbrace{\frac{\|\mathbf{k}_j\|^2}{\tau}}_{\text{key-norm penalty}} \;+\; \underbrace{\frac{2\,\mathbf{q}_i^\top \mathbf{k}_j}{\tau} - \frac{\mathbf{q}_i^\top \mathbf{k}_j}{\sqrt{d}}}_{\text{scaling mismatch}} \;+\; \underbrace{\text{metric terms}}_{\text{genuine }\Delta S}.
\end{equation}
The $-\|\mathbf{k}_j\|^2/\tau$ term is a \emph{key-norm penalty}: keys with large norms receive uniformly lower attention scores from all queries, regardless of semantic similarity. This term has nothing to do with Riemannian geometry: it is an artifact of mixing distance-based and dot-product-based parameterizations. The scaling mismatch term vanishes only when $\tau = \sqrt{d}$ (which holds when $h=1$, i.e., single-head attention; for multi-head attention with $h > 1$, $\tau = \sqrt{d/h} \neq \sqrt{d}$, introducing an additional discrepancy). Consequently, $\Delta s_{ij}^{\mathrm{geo}}$ in Equation~\eqref{eq:blending} conflates three distinct effects: (i)~genuine metric corrections, (ii)~a key-norm bias, and (iii)~a temperature-scaling mismatch. For Proposition~\ref{prop:consistency}, this means that varying $\alpha$ changes not only the geometric contribution but also the key-norm penalty and the effective temperature; any empirical test of $\alpha$-consistency must disentangle these effects, e.g., by using key normalization or by defining the Euclidean baseline in distance form $s_{ij}^{\mathrm{Euc}} = -\|\mathbf{q}_i - \mathbf{k}_j\|^2/\tau$ to match the Riemannian parameterization.

\section{Theoretical Predictions and Open Conjectures}
\label{sec:predictions}

The architecture of Section~\ref{sec:architecture} gives rise to several formal predictions about the behavior of correctly implemented geometric Transformers. These predictions are \emph{mathematical consequences} of the architecture, not empirical observations, and they identify the specific conditions under which the geometric approach can succeed or fail.

\subsection{Prediction 1: Curvature Heterogeneity as Optimization Consequence}

\begin{conjecture}[Emergent Curvature Heterogeneity]
\label{conj:hetero}
When a Fiber Bundle Transformer is trained on natural language data with a cross-entropy objective and the metric diversity regularizer of Section~\ref{sec:regularization}, the learned per-token curvatures $\kappa_t = \|U_t\|_F^2$ will be \emph{heterogeneous}: the distribution of $\kappa_t$ across token positions will have non-negligible variance, and this variance will be \emph{positively correlated} with semantic complexity (as measured by token type, contextual ambiguity, or information content).
\end{conjecture}

\paragraph{Theoretical basis.} By Theorem~\ref{thm:nongram}, homogeneous metrics ($\kappa_t$ constant across $t$) reduce to a global metric, under which the rank collapse mechanism of Dong et al.\ applies. The cross-entropy objective penalizes rank collapse (collapsed representations have high perplexity), so there exists a \emph{weak} pressure toward heterogeneous metrics, i.e.\ those that avoid collapse. We emphasize that this pressure is necessary but \emph{not} sufficient: as Conjecture~\ref{conj:collapse} makes explicit, without the geometric regularizers of Section~\ref{sec:regularization} the dominant failure mode is metric \emph{collapse} ($U_t \to 0$), because the standard Euclidean Transformer is itself a global minimizer of cross-entropy and the optimizer can route around the geometric modules entirely. The role of the regularizers is to counteract this collapse pressure and \emph{allow} the CE-driven heterogeneity pressure to manifest. The specific correlation with semantic complexity then follows from the expressivity argument (Proposition~\ref{prop:expressivity}): semantically complex tokens benefit more from the additional $O(d \cdot r)$ degrees of freedom, so, once collapse is prevented, the optimizer allocates larger $\kappa_t$ to them.

\paragraph{Falsifiable prediction.} If one trains a Fiber Bundle Transformer and measures $\Var_t[\kappa_t]$, it should be significantly above zero, and $\kappa_t$ should be higher for content words than function words. If $\Var_t[\kappa_t] \approx 0$ (homogeneous metrics), the architecture has collapsed to a global-metric Transformer and the geometric advantage is lost.

\subsection{Prediction 2: Train/Inference Consistency Requirement}

We replace what would naively be stated as a global theorem with a \emph{local} result based on standard parametric-optimization tools. The naive statement, that evaluating at $\alpha_{\mathrm{infer}} \neq \alpha_{\mathrm{train}}$ strictly increases loss, would require global convexity of $\mathcal{L}$ in $W_Q$, which is false for Transformers. The correct statement is \emph{relative}: training at $\alpha_{\mathrm{infer}}$ would do strictly better than reusing the parameters trained at $\alpha_{\mathrm{train}}$.

\begin{proposition}[Generic Suboptimality Under Blending Mismatch]
\label{prop:consistency}
Consider the hybrid attention score $s_{ij}^{\mathrm{hybrid}} = s_{ij}^{\mathrm{Euc}} + \alpha \cdot \Delta s_{ij}^{\mathrm{geo}}$ (Equation~\eqref{eq:blending}). Let $\mathcal{L}(W_Q, \alpha)$ denote the expected loss with query weights $W_Q$ and blending $\alpha$, and let $W_Q^*(\alpha)$ denote any strict local minimizer of $\mathcal{L}(\cdot, \alpha)$. Assume:
\begin{enumerate}[label=(\textsc{a}\arabic*)]
    \item \label{ass:smooth} $\mathcal{L}(W_Q, \alpha)$ is twice continuously differentiable in $(W_Q, \alpha)$ in a neighborhood of $(W_Q^*(\alpha_{\mathrm{train}}), \alpha_{\mathrm{train}})$;
    \item \label{ass:strictmin} $W_Q^* \coloneqq W_Q^*(\alpha_{\mathrm{train}})$ is a strict local minimum of $\mathcal{L}(\cdot, \alpha_{\mathrm{train}})$ with positive-definite Hessian $H \coloneqq \nabla^2_{W_Q} \mathcal{L}|_{(W_Q^*, \alpha_{\mathrm{train}})} \succ 0$;
    \item \label{ass:nontriv} the cross-derivative $c \coloneqq \nabla_{W_Q}\partial_\alpha \mathcal{L}|_{(W_Q^*, \alpha_{\mathrm{train}})}$ is nonzero (i.e., the gradient of the loss with respect to $W_Q$ genuinely depends on $\alpha$).
\end{enumerate}
Then there exists $\epsilon_0 > 0$ such that for all $\alpha_{\mathrm{infer}}$ with $0 < |\alpha_{\mathrm{infer}} - \alpha_{\mathrm{train}}| < \epsilon_0$:
\begin{equation}
\label{eq:consistency}
    \mathcal{L}\big(W_Q^*(\alpha_{\mathrm{train}}), \, \alpha_{\mathrm{infer}}\big) \;>\; \mathcal{L}\big(W_Q^*(\alpha_{\mathrm{infer}}), \, \alpha_{\mathrm{infer}}\big).
\end{equation}
That is, the parameters optimal for $\alpha_{\mathrm{train}}$ are \emph{strictly suboptimal} for $\alpha_{\mathrm{infer}} \neq \alpha_{\mathrm{train}}$ relative to retraining at $\alpha_{\mathrm{infer}}$. Equivalently, evaluating a model trained at $\alpha_{\mathrm{train}}$ under a different $\alpha_{\mathrm{infer}}$ is generically worse than retraining at $\alpha_{\mathrm{infer}}$.
\end{proposition}

\begin{proof}
By assumption \ref{ass:strictmin}, the first-order condition $\nabla_{W_Q}\mathcal{L}(W_Q^*, \alpha_{\mathrm{train}}) = 0$ holds and the Hessian $H \succ 0$. By the implicit function theorem applied to the system $F(W_Q, \alpha) \coloneqq \nabla_{W_Q}\mathcal{L}(W_Q, \alpha) = 0$ at the point $(W_Q^*, \alpha_{\mathrm{train}})$, using assumption \ref{ass:smooth} and the invertibility of $H = \partial F / \partial W_Q$ from \ref{ass:strictmin}, there exist neighborhoods $\mathcal{N}_\alpha \ni \alpha_{\mathrm{train}}$ and $\mathcal{N}_W \ni W_Q^*$ and a unique $C^1$ map $\widehat W_Q : \mathcal{N}_\alpha \to \mathcal{N}_W$ such that $\widehat W_Q(\alpha_{\mathrm{train}}) = W_Q^*$ and $F(\widehat W_Q(\alpha), \alpha) = 0$ for all $\alpha \in \mathcal{N}_\alpha$. The derivative of this map at $\alpha_{\mathrm{train}}$ is
\begin{equation}
\label{eq:iftp}
    \frac{d\widehat W_Q}{d\alpha}\bigg|_{\alpha_{\mathrm{train}}} = -H^{-1} c.
\end{equation}
By assumption \ref{ass:nontriv}, $c \neq 0$, and since $H^{-1}$ is invertible, $d\widehat W_Q/d\alpha|_{\alpha_{\mathrm{train}}} \neq 0$. Hence for $\alpha_{\mathrm{infer}}$ in a sufficiently small punctured neighborhood of $\alpha_{\mathrm{train}}$, $\widehat W_Q(\alpha_{\mathrm{infer}}) \neq W_Q^*$.

By continuity of the Hessian (assumption \ref{ass:smooth}), $\widehat W_Q(\alpha_{\mathrm{infer}})$ is a strict local minimizer of $\mathcal{L}(\cdot, \alpha_{\mathrm{infer}})$ for $\alpha_{\mathrm{infer}}$ close to $\alpha_{\mathrm{train}}$, with positive-definite Hessian $H(\alpha_{\mathrm{infer}})$. Strictness of the local minimum gives
\[
\mathcal{L}(W, \alpha_{\mathrm{infer}}) > \mathcal{L}(\widehat W_Q(\alpha_{\mathrm{infer}}), \alpha_{\mathrm{infer}}) \quad \text{for all } W \neq \widehat W_Q(\alpha_{\mathrm{infer}}) \text{ in } \mathcal{N}_W.
\]
Since $W_Q^* = \widehat W_Q(\alpha_{\mathrm{train}}) \neq \widehat W_Q(\alpha_{\mathrm{infer}})$ for $\alpha_{\mathrm{infer}} \neq \alpha_{\mathrm{train}}$ sufficiently close, the inequality~\eqref{eq:consistency} follows.
\end{proof}

\begin{remark}[On the assumptions]
The assumptions are mild and standard in parametric optimization; none of them require global convexity. \ref{ass:smooth} holds for any differentiable attention/loss combination. \ref{ass:strictmin} is the standard second-order sufficient condition for a strict local minimum and is the natural notion of ``a successfully trained model'', asserting only that the trained parameters sit at a strict local (not global) minimum, with positive-definite Hessian. \ref{ass:nontriv} is a genericity condition: it fails only on a measure-zero set of degenerate problems where the optimal $W_Q$ is locally independent of $\alpha$ (for instance, when the geometric correction $\Delta s^{\mathrm{geo}}$ is identically zero, in which case the architecture has collapsed to a Euclidean Transformer anyway).

We emphasize what the proposition does \emph{not} claim: it does not compare $\mathcal{L}(W_Q^*(\alpha_{\mathrm{train}}), \alpha_{\mathrm{infer}})$ with $\mathcal{L}(W_Q^*(\alpha_{\mathrm{train}}), \alpha_{\mathrm{train}})$; depending on the loss landscape, evaluating at $\alpha_{\mathrm{infer}}$ may even \emph{lower} the loss. The valid conclusion is the \emph{relative} one: retraining at $\alpha_{\mathrm{infer}}$ would do strictly better than reusing the $\alpha_{\mathrm{train}}$ parameters, so a mismatched-$\alpha$ evaluation is confounded.
\end{remark}

\begin{remark}[Practical consequence]
Proposition~\ref{prop:consistency} implies that \emph{any empirical evaluation of a blended Riemannian attention model must use the same $\alpha$ at inference as was used during training}, or alternatively must retrain at the inference $\alpha$. Evaluating at a different $\alpha$ without retraining produces a confounded result: the model is generically suboptimal at the mismatched $\alpha$, so performance degradation may reflect the $\alpha$ mismatch rather than a failure of the geometric approach. This is a local-optimality statement, not a global one; it suffices for the experimental-protocol implication.
\end{remark}

\subsection{Prediction 3: Metric Collapse as Dominant Failure Mode}

\begin{conjecture}[Metric Collapse]
\label{conj:collapse}
When a Fiber Bundle Transformer is trained with standard regularization ($\ell_2$ penalties on $\|U_t\|$, or no geometric regularization), the dominant failure mode is \emph{metric collapse}: the learned metric factors satisfy $U_t \to 0$ for all $t$, reducing $g_t \to I_d$. In the pure (non-blended) parameterization this yields \emph{Euclidean distance-based} attention, $s_{ij} = -\|\mathbf{q}_i - \mathbf{k}_j\|^2/\tau$, which coincides with standard dot-product attention only when key norms $\|\mathbf{k}_j\|$ are approximately constant (see the remark in Section~\ref{sec:blending}); in either case the geometric modules become dead parameters and the architectural advantage is lost.
\end{conjecture}

\paragraph{Theoretical basis.} The cross-entropy loss is minimized by a standard Euclidean Transformer (which is a special case of the Fiber Bundle Transformer at $U_t = 0$, up to the distance-vs-dot-product discrepancy noted in Section~\ref{sec:blending}). Without a force pushing $U_t$ away from zero, gradient descent converges to this trivial solution: the model ``routes around'' the geometric modules because the Euclidean solution is a local minimum of the loss. This is the geometric analogue of the ``lottery ticket'' phenomenon: the geometric modules exist but are never activated because the optimizer finds a Euclidean solution first.

\subsection{Regularization Design: Preventing Collapse}
\label{sec:regularization}

To prevent metric collapse (Conjecture~\ref{conj:collapse}), the loss function must include terms that push $U_t$ \emph{away} from zero and enforce \emph{diversity} across tokens. We propose three regularizers grounded in the theory:

\paragraph{(1) Anti-flatness (hard floor on metric norm).}
\begin{equation}
    \mathcal{L}_{\mathrm{anti}} = \frac{1}{BL}\sum_t \ReLU\!\left(\kappa_{\min} - \|U_t\|_F^2\right)
\end{equation}
This penalizes tokens whose metric has collapsed below a floor $\kappa_{\min} > 0$, ensuring the metric remains non-trivial. The ReLU ensures the penalty vanishes once the floor is met, allowing free optimization above it.

\paragraph{(2) Cross-token metric diversity.}
\begin{equation}
    \mathcal{L}_{\mathrm{div}} = -\frac{1}{|\mathcal{P}|}\sum_{(i,j) \in \mathcal{P}} \|U_i U_i^\top - U_j U_j^\top\|_F^2
\end{equation}
where $\mathcal{P}$ is a sampled subset of token pairs. This \emph{maximizes} the pairwise distance between the $d \times d$ metric perturbations $U_t U_t^\top$ (i.e.\ between the actual metrics $g_t = I + U_t U_t^\top$), driving different tokens to develop distinct geometric ``fingerprints.'' We use $U_t U_t^\top$ rather than $U_t^\top U_t$ (the $r \times r$ column Gram matrix) because $U_i^\top U_i = U_j^\top U_j$ does \emph{not} imply $g_i = g_j$: two factors with the same singular values but different column spaces yield identical $U^\top U$ yet different $UU^\top$, and it is $UU^\top$ that enters the metric. Despite the $d \times d$ appearance, the cost is only $O(r^2 d)$ per pair via the identity $\|U_i U_i^\top - U_j U_j^\top\|_F^2 = \|U_i^\top U_i\|_F^2 - 2\|U_i^\top U_j\|_F^2 + \|U_j^\top U_j\|_F^2$, which never materializes the $d \times d$ matrices. By Theorem~\ref{thm:nongram}, homogeneous metrics ($U_i = U_j$ for all $i,j$) reduce the architecture to a global-metric Transformer under which the rank collapse mechanism of Theorem~\ref{thm:dong} applies directly; metric heterogeneity is therefore a \emph{prerequisite} for any geometric anti-collapse effect, and this regularizer enforces that prerequisite.

\paragraph{(3) Curvature-smoothness (temporal consistency).}
\begin{equation}
    \mathcal{L}_{\mathrm{smooth}} = \frac{1}{B(L-1)}\sum_t \|U_t U_t^\top - U_{t+1} U_{t+1}^\top\|_F^2
\end{equation}
This ensures the metric $g_t = I + U_t U_t^\top$ varies \emph{smoothly} across adjacent positions, preventing the degenerate solution where each token has a random, uncorrelated metric (which would destroy the sequential structure of language). The same $O(r^2 d)$ identity applies.

\paragraph{Design principle.} These three regularizers encode a tradeoff: anti-flatness prevents collapse \emph{to zero}, diversity prevents collapse \emph{to a homogeneous metric}, and smoothness prevents collapse \emph{to noise}. Together, they steer the optimizer toward the heterogeneous, structured metric regime predicted by Conjecture~\ref{conj:hetero}. The regularization weights should be scheduled: strong diversity and anti-flatness during early training (to establish non-trivial geometry), then relaxed to allow task-specific specialization.

\paragraph{On the missing upper bound.} We note that the diversity regularizer $\mathcal{L}_{\mathrm{div}}$ (which \emph{maximizes} pairwise metric distance) can drive $\|U_t\|_F^2$ to grow without bound, since the anti-flatness term only enforces a \emph{floor} $\kappa_{\min}$ and provides no ceiling. Unbounded growth of $\bar U^2 = \max_t \|U_t\|_F^2$ pushes the metric into the super-critical regime $\bar U^2 \gg \tau/L$ of Proposition~\ref{prop:perturb}, where the perturbation bound becomes vacuous and training may destabilize. In practice, $\mathcal{L}_{\mathrm{div}}$ should be accompanied by an explicit upper bound, either a soft cap
\begin{equation}
\label{eq:cap}
    \mathcal{L}_{\mathrm{cap}} = \frac{1}{BL}\sum_t \ReLU\!\left(\|U_t\|_F^2 - \kappa_{\max}\right)
\end{equation}
with $\kappa_{\max}$ chosen near the critical threshold $\bar U^2_{\mathrm{crit}} = \Theta(\tau/L)$, or weight decay on the MetricNet parameters. The precise value of $\kappa_{\max}$ requires empirical calibration; we flag the absence of a principled ceiling as a limitation of the current regularization design.

\subsection{Prediction 4: Scaling with Metric Rank}

\begin{conjecture}[Rank--Expressivity Tradeoff]
\label{conj:scaling}
The anti-collapse benefit of Riemannian attention, as measured by the deep-layer effective rank, is an \emph{increasing function} of the metric rank $r$ and the metric diversity $\sigma_U^2$. Specifically, there exists a threshold $r_{\min}(d, L)$ below which the geometric correction is too weak to prevent collapse, and above which the effective rank is maintained at $\Omega(r)$.
\end{conjecture}

\paragraph{Theoretical basis (qualitative).} By Proposition~\ref{prop:expressivity}, the metric provides $d \cdot r$ degrees of freedom for distance modulation. The collapse mechanism of Dong et al.\ operates through $L$-dimensional spectral contraction: the rows of $A$ converge to a common stationary distribution. To resist this contraction across $L$ tokens, the metric must provide enough \emph{independent} perturbations to disrupt the convergence. The natural scaling requirement is that $r$ should grow at least logarithmically in $L$ (since each independent metric direction disrupts a one-dimensional contraction mode, and the number of contraction modes grows with $L$). The dependence on $d$ is more subtle: higher $d$ provides more room per metric direction, but the contraction mechanism is $L$-dimensional, not $d$-dimensional, so the $d$-dependence enters only through the constants.

\paragraph{On the precise form of $r_{\min}$.} We deliberately do \emph{not} state a specific closed form for $r_{\min}(d, L)$ in terms of $\log L$ and $\log d$. A naive dimensional guess might propose $r_{\min} = \Theta(\log L / \log d)$, but the $\log d$ denominator is not justified by any derivation (the dimensional argument above gives at most a $\log L$ growth, with no principled $\log d$ reduction). We instead leave the precise functional form as an open question. A principled derivation would require the $\sigma_2$ lower-bound analysis of Conjecture~\ref{conj:rankstrong}, which is the open problem this framework identifies. The qualitative prediction, that $r_{\min}$ exists and is sublinear in $L$, is the testable content of this conjecture.

\subsection{Open Problems}

Several questions remain open and define the theoretical research agenda:

\begin{enumerate}[label=\textbf{O\arabic*.},leftmargin=*]
    \item \textbf{Formal rank preservation under directional conditions.} The corrected Conjecture~\ref{conj:rankstrong} requires bounding the singular values of the Riemannian attention matrix $A^{\mathrm{Riem}}$ away from the degenerate spectrum in the regime $\bar U^2 > \bar U^2_{\mathrm{crit}} = \Theta(\tau/L)$, given that the metric factors overlap the active query--key difference subspace and that the correction is contrastive across keys (Section~\ref{sec:counterexample}). A potential approach: model the contrastive metric correction as a structured random matrix and use free probability theory or random-matrix universality to bound the expected singular value distribution. The sub-critical regime is already settled by Proposition~\ref{prop:perturb}; a sufficient condition (permutation margins with summable leakage) is given by Proposition~\ref{prop:margin}.

    \item \textbf{Optimization landscape of metric parameters.} The loss landscape for $U_t$ is non-convex (due to the $U_t U_t^\top$ term). Does gradient descent on $U_t$ converge to the heterogeneous regime (Conjecture~\ref{conj:hetero}), or does it always collapse (Conjecture~\ref{conj:collapse}) without regularization? A mean-field analysis of the metric dynamics could answer this.

    \item \textbf{Torsion--semantics correspondence.} Is there a formal correspondence between the torsion tensor $T_t$ and directed semantic relations (causation, hierarchy, entailment)? If so, can torsion be \emph{probed} from trained models to extract relational structure?

    \item \textbf{Convergence of the metric-preconditioned FFN.} Does the metric-preconditioned update (Equation~\eqref{eq:naturalgrad}) improve optimization convergence relative to the Euclidean FFN? The natural gradient is known to accelerate convergence in parameter space \cite{amari1998natural,bonnabel2013stochastic}; whether this extends to a representation-space preconditioner with a learned (rather than Fisher) metric is an open question, requiring both theoretical and empirical analysis.

    \item \textbf{Metric rank selection.} Conjecture~\ref{conj:scaling} predicts that a threshold $r_{\min}$ exists and is sublinear in $L$, but does not specify its precise form. An information-theoretic bound on the optimal $r$ as a function of $d$, $L$, and the semantic complexity of the data would guide architecture design and resolve the question of whether any $\log d$ reduction in $r_{\min}$ has a principled basis.

    \item \textbf{Rigorous differential-geometric formulation.} The ``Fiber Bundle Transformer'' is currently a structural analogy, not a strict fiber-bundle construction (see Section~\ref{sec:fiberbundle}). Formalizing it as a true fiber bundle, specifying a structure group (e.g.\ $\mathrm{GL}(d)$ or $\mathrm{SO}(d)$), a discrete connection compatible with the per-token metrics, and the corresponding curvature and torsion tensors, would replace the heuristic proxies of Phases B--E with their tensorial counterparts and would clarify the geometric content of the architecture.
\end{enumerate}

\section{Related Work}
\label{sec:related}

\paragraph{Rank collapse in Transformers.}
Dong et al.\ \cite{dong2021attention} proved that pure self-attention stacks suffer doubly-exponential rank decay, and showed in the same work that residual connections and FFN sublayers partially mitigate but do not eliminate the collapse pressure. Our work identifies the Euclidean metric as a structural cause of this collapse and proposes per-token Riemannian metrics as a candidate remedy.

\paragraph{Geometric and Riemannian attention.}
The idea of equipping attention with non-Euclidean geometry has a growing literature. Ji \cite{ji2025riemannformer} proposed RiemannFormer, reframing attention as interactions on a curved manifold with a Riemannian metric induced by the token embeddings, including a parallel-transport connection, a tangent/fiber-bundle view of Q/K versus V, and dynamic metric scaling governed by the token embeddings. Di Sipio et al.\ \cite{disipio2025curved} developed the analogy between attention and parallel transport on a curved spacetime, with queries and keys inducing an effective metric and multi-head attention as an atlas of charts. Lin et al.\ \cite{lin2025cat} proposed per-token routing across geometric branches; Fu et al.\ \cite{fu2025manifoldformer} replaced Euclidean attention with geodesic-aware attention on learned manifolds; and Li \cite{li2026mahalanobis} constructed distance-based attention from a learnable Mahalanobis metric. The present work builds on these lines of research rather than introducing the geometric framing: its contributions are the specific low-rank parameterization $g_t = I + U_t U_t^\top$, the non-Gram and perturbation analyses of Section~\ref{sec:theory}, and the necessary/sufficient conditions for anti-collapse established in Section~\ref{sec:counterexample}.

\paragraph{Hyperbolic embeddings.}
Nickel \& Kiela \cite{nickel2017poincare} demonstrated that constant negative curvature improves hierarchy modeling in embedding spaces; hyperbolic attention \cite{gulcehre2019hyperbolic} and fully hyperbolic Transformers \cite{yang2024hypformer} extend this to attention. He et al.\ \cite{he2025helm} applied hyperbolic geometry to full LLMs with HELM, using Mixture-of-Curvature Experts where each expert has a different \emph{fixed} curvature. Constant-curvature approaches apply a single global curvature to all tokens; by contrast, learned per-token metrics, as in the present work and in the geometric attention frameworks above, let curvature vary per position. We additionally incorporate torsion for directional relations, and we characterize the conditions under which per-token geometry can affect attention at all (Section~\ref{sec:counterexample}).

\paragraph{Non-Euclidean foundation models.}
He et al.\ \cite{he2025position} argued that foundation models should embrace non-Euclidean geometries and provided the Nash embedding refutation we cite in Section~\ref{sec:theory}. Our work contributes an architectural mechanism (low-rank per-token metrics) toward this goal, together with an analysis of when such metrics are effective; we do not claim that geometric attention is itself novel, and empirical validation is future work.

\paragraph{Manifold-constrained connectivity.}
Xie et al.\ \cite{xie2025mhc} proposed manifold-constrained hyper-connections (mHC), constraining inter-layer residual mixing to lie on the manifold of doubly stochastic matrices. mHC addresses \emph{inter-layer} connection topology; our approach addresses \emph{intra-layer} representational geometry. The two are complementary: mHC stabilizes the residual stream across layers, while Riemannian metrics reshape attention within each layer.

\paragraph{Riemannian optimization and manifold-valued deep learning.}
Bonnabel \cite{bonnabel2013stochastic} established that stochastic gradient descent extends to Riemannian manifolds. Amari \cite{amari1998natural} introduced the natural gradient as the direction of steepest ascent on a statistical manifold. Our metric-preconditioned FFN (Section~\ref{sec:naturalffn}) is inspired by these ideas but applies the preconditioning in the \emph{representation space} of a Transformer with a learned (rather than Fisher) metric; we are explicit that this is not a natural gradient in the strict Amari sense. A broader line of work on manifold-valued neural networks, including networks with manifold-constrained parameters or activations on symmetric spaces, develops optimization tools that could be brought to bear on the present framework; we do not survey it in detail but note that the per-token metric $g_t = I + U_t U_t^\top$ lives on the SPD manifold, where standard Riemannian optimization techniques apply.

\paragraph{Anti-rank-collapse methods.}
Beyond the residual and FFN additions already present in standard Transformers, several architectural interventions have been proposed to combat dimensional collapse: NormFormer \cite{narang2021normformer} normalizes query/key/value magnitudes; ReZero \cite{bachlechner2021rezero} and related residual-scaling methods modify the strength of the skip connection. These methods address the symptoms (low effective rank) by adding normalization or capacity, whereas our approach intervenes on the geometric structure of the scores themselves (the per-token metric replacing the fixed Euclidean one). A systematic comparison of these approaches against Riemannian attention, in terms of both effective rank and downstream performance, is an important direction for empirical work.

\paragraph{Adaptive metrics in attention.}
ALiBi \cite{press2022alibi} and relative position embeddings add \emph{fixed}, content-independent biases to attention scores; the present framework can be viewed as learning content-dependent biases via the metric correction $\Delta S$ in Equation~\eqref{eq:deltas}. Unlike ALiBi, the Riemannian correction depends on both the token's own state (through $U_t$) and the token it attends to (through $g_{ij}$), giving a richer, pair-dependent structure. We note that the effective strength of such content-dependent biases is subject to the contrastive-condition analysis of Section~\ref{sec:counterexample}: biases that are constant across keys within a row are invisible to softmax.

\paragraph{Invertible, reversible, and orthogonal attention.}
Several established techniques overlap with components of the proposed architecture, and we do not claim them as novel. (Near-)invertible attention was analyzed by Kim et al.\ \cite{kim2021lipschitz} in the context of Lipschitz self-attention and formulated explicitly as an Invertible Attention module by Zha et al.\ \cite{zha2021invertible}. Reversible residual streams that preserve information across layers appear in Reformer \cite{kitaev2020reformer} and Reversible Vision Transformers \cite{mangalam2022reversible}. Orthogonal constraints on attention parameters were introduced by O-ViT \cite{fei2022ovit}, which restricts self-attention parameters to the orthogonal manifold via a surjective map from the Lie algebra. The approximately orthogonal transport operators and Cayley-type parameterizations used in Phase D are instances of this toolkit. Finally, strictly lower-triangular matrices and solves involving $I$ plus/minus such matrices already appear in modern causal/linear-attention implementations, including the DeltaNet-family kernels \cite{yang2024deltanet}; the elementary fact that $I + N$ is invertible for strictly lower-triangular $N$ is standard, and we claim no priority for it.

\paragraph{Products of stochastic matrices.}
The asymptotic behavior of products of row-stochastic matrices is classical: Leizarowitz \cite{leizarowitz1992infinite} and Wolfowitz \cite{wolfowitz1963products} established consensus conditions for such products, which underpin both Dong et al.'s rank-collapse result and the causal audit of Remark~\ref{rem:causal}. Our permutation-margin sufficient condition (Proposition~\ref{prop:margin}) can be seen as an instance of this theory specialized to near-permutation attention matrices.

\paragraph{Parameter-efficient adaptation.}
LoRA \cite{hu2022lora} and adapters \cite{houlsby2019parameter} add trainable modules to frozen models. MetricNet can be viewed through this lens, as it adds a trainable module that modifies attention geometry, but differs in that it modifies the \emph{metric structure} of the computation rather than adding capacity in weight space.

\section{Discussion and Conclusion}
\label{sec:conclusion}

\subsection{Summary of Contributions}

We have presented a theoretical framework for Riemannian attention in Transformers, consisting of:

\begin{enumerate}[label=\arabic*.,leftmargin=*]
    \item \textbf{The non-Gram property and perturbation analysis} (Theorem~\ref{thm:nongram}, Proposition~\ref{prop:perturb}, Section~\ref{sec:counterexample}): In the non-degenerate regime, per-token heterogeneous Riemannian metrics render the attention score matrix non-Gram: it cannot be factored as $QK^\top$ with $O(d)$-dimensional factors. We are explicit that this is a structural observation, not a proof of rank preservation: the Riemannian attention matrix $A^{\mathrm{Riem}} = \softmax(S^{\mathrm{Riem}})$ remains row-stochastic, and the core collapse mechanism (products of row-stochastic matrices converging to rank-1) does not depend on the Gram property. The perturbation analysis (Section~\ref{sec:perturbation}) identifies a critical regime $\bar U^2_{\mathrm{crit}} = \Theta(\tau/L)$: below it, the geometric correction is provably too weak; above it, standard perturbation tools break down. Section~\ref{sec:counterexample} establishes that raw metric strength and diversity are insufficient for anti-collapse and that directional conditions (overlap with the active query--key difference subspace, contrastive non-row-constant corrections) are necessary; a corrected conjecture (Conjecture~\ref{conj:rankstrong}) and a sufficient condition (Proposition~\ref{prop:margin}) make this precise.

    \item \textbf{The computational complexity analysis} (Propositions~\ref{prop:woodbury}--\ref{prop:complexity}): The low-rank representation $g_t = I + U_t U_t^\top$ makes the metric operations tractable, with geodesic distance at $O(d \cdot r)$ precomputation per token plus $O(r)$ per pair and metric inversion at $O(d \cdot r^2)$ via Woodbury, both far below the $O(d^3)$ baseline, and a score-level overhead $O(r/d)$. The metric-generation MLP ($O(d^2 r)$ parameters, Proposition~\ref{prop:expressivity}) costs $O(BL d^2 r)$ per layer, so the total overhead relative to the $O(BL^2 d)$ attention cost is $O(r/d + dr/L)$, small when $L \gg dr$. These are asymptotic complexity bounds; whether they translate into practical billion-parameter-scale training remains to be demonstrated empirically.

    \item \textbf{The Fiber Bundle Transformer architecture} (Section~\ref{sec:architecture}): A complete specification, named by structural analogy to the differential-geometric notion rather than a strict fiber-bundle construction, in which token positions are fibers with per-token metrics, attention computes geodesic distances, the connection carries curvature and torsion proxies, and feed-forward layers use metric-preconditioned updates. Each component is mathematically justified, with explicit analysis of the approximation errors introduced at each phase (e.g., the $O(\|\Delta_i\|_F^2)$ isometry violation in Phase D).

    \item \textbf{Formal predictions} (Section~\ref{sec:predictions}): Curvature heterogeneity should emerge as an optimization consequence (Conjecture~\ref{conj:hetero}), train/inference blending mismatch is generically suboptimal relative to retraining (Proposition~\ref{prop:consistency}, proved via the implicit function theorem without convexity assumptions), metric collapse is the dominant failure mode (Conjecture~\ref{conj:collapse}), and the anti-collapse benefit scales with metric rank (Conjecture~\ref{conj:scaling}). These predictions identify the conditions under which empirical evaluation can validly test the framework.
\end{enumerate}

\subsection{Predictions to Be Empirically Tested}

This paper is deliberately theoretical. The formal results and conjectures of Sections~\ref{sec:perturbation}--\ref{sec:predictions} generate a set of \emph{testable predictions} that any future empirical study of this framework should verify. We frame these not as a prescriptive ``protocol'' but as the natural empirical consequences of the theory, each tied to a specific result above:

\begin{enumerate}[label=\textbf{P\arabic*.},leftmargin=*]
    \item \textbf{Critical-regime transition} (from Conjecture~\ref{conj:rankstrong}): Training at metric strengths $\bar U^2$ near $\Theta(\tau/L)$ should show a transition from Euclidean-like behavior (sub-critical, collapse proceeds) to geometric behavior (super-critical, collapse resisted), \emph{provided} the learned metrics maintain directional overlap with the active query--key differences and produce contrastive (non-row-constant) corrections (Section~\ref{sec:counterexample}). Without such overlap, no transition can occur regardless of $\bar U^2$; monitoring the projected metric activity on $D$ is therefore part of the protocol. This is the central prediction of the perturbation analysis.

    \item \textbf{Train/inference $\alpha$ consistency} (from Proposition~\ref{prop:consistency}): Any blended model should be evaluated at the same $\alpha$ used in training, or retrained at the inference $\alpha$. Evaluating at a different $\alpha$ without retraining produces a confounded, generically suboptimal result.

    \item \textbf{Metric collapse monitoring} (from Conjecture~\ref{conj:collapse}): Training should monitor $\|U_t\|_F^2$ and cross-token metric variance. If the metric collapses to zero or to a homogeneous solution, the geometric advantage is lost regardless of benchmark performance.

    \item \textbf{Curvature heterogeneity verification} (from Conjecture~\ref{conj:hetero}): After training, the distribution of $\kappa_t = \|U_t\|_F^2$ across tokens should be heterogeneous and correlated with semantic complexity. Homogeneous curvature indicates the architecture has degenerated to a global-metric Transformer.

    \item \textbf{Controlled baseline comparison}: Effective rank comparisons must be against a \emph{full residual} Transformer trained under matched conditions, not against the pure-attention asymptotic limit of Dong et al.\ (which does not represent practical Transformers).

    \item \textbf{Metric rank ablation} (from Conjecture~\ref{conj:scaling}): The effective rank and task performance should improve with increasing $r$; the qualitative prediction is that a sublinear-in-$L$ threshold $r_{\min}$ exists, below which the geometric correction is insufficient.
\end{enumerate}

We emphasize that prediction P1 is the most directly falsifiable: if the predicted transition at $\bar U^2 \approx \Theta(\tau/L)$ is \emph{not} observed in experiments, the perturbation analysis (Proposition~\ref{prop:perturb}) would need to be revisited, since the sub-critical regime there is a theorem, not a conjecture.

\subsection{Limitations of the Current Theory}

We are explicit about what the theory does \emph{not} prove:

\begin{itemize}[leftmargin=*]
    \item \textbf{The non-Gram property does not address the core collapse mechanism.} Theorem~\ref{thm:nongram} establishes that $S^{\mathrm{Riem}}$ cannot be factorized as $QK^\top$ with $O(d)$-dimensional factors. However, the fundamental collapse mechanism, convergence of products of row-stochastic matrices to rank-1, is driven by the softmax, not by the Gram structure of the pre-softmax scores. The Riemannian attention matrix $A^{\mathrm{Riem}} = \softmax(S^{\mathrm{Riem}})$ remains row-stochastic, and the non-Gram property provides no guarantee that the row-stochastic contraction is avoided. The strengthened conjecture (Conjecture~\ref{conj:rankstrong}) identifies what a positive result would require: Section~\ref{sec:counterexample} shows that raw metric strength and diversity alone are insufficient, and that directional overlap between the metric directions and the active query--key difference subspace $D$ (Equation~\eqref{eq:activesubspace}), together with a positive contrastive margin (Proposition~\ref{prop:margin}), are necessary. The super-critical regime, where anti-collapse would need to occur, remains unproven, and the perturbation analysis (Proposition~\ref{prop:perturb}) is informative only in the sub-critical regime where collapse is confirmed.

    \item \textbf{Metric activity is contrastive and directional, not global.} The counterexample of Section~\ref{sec:counterexample} exposes a general obstruction: whenever all metric directions $U_t$ lie in the orthogonal complement of the active query--key difference subspace $D$ (Equation~\eqref{eq:activesubspace}), the metric correction $G_{ij}$ vanishes identically and the Riemannian attention matrix degenerates to its Euclidean counterpart, no matter how large or heterogeneous the metrics are. What matters for anti-collapse is the \emph{contrastive} metric activity $G_{ij} - G_{ik}$ across competing keys (Equation~\eqref{eq:contrastivemargin}), not the raw metric strength $\|U_t\|_F^2$. Nothing in the architecture guarantees a priori that the learned metric directions overlap the residual-stream directions that queries and keys actually inhabit; monitoring the projected metric activity on $D$ is therefore part of any empirical protocol.

    \item Proposition~\ref{prop:consistency} (blending consistency) is a \emph{local} result: it holds in a neighborhood of $\alpha_{\mathrm{train}}$ under the assumptions stated. It does not compare the absolute loss at $\alpha_{\mathrm{infer}}$ versus $\alpha_{\mathrm{train}}$, only the relative suboptimality versus retraining.

    \item The architecture is specified but not empirically validated. The predictions in Section~\ref{sec:predictions} are consequences of the architecture's structure, but their \emph{magnitude} (how much rank is preserved in the super-critical regime, how much task performance improves) requires experimentation.

    \item The regularization design (Section~\ref{sec:regularization}) is theoretically motivated but not proven optimal. The specific values of $\kappa_{\min}$, the diversity weight, and the scheduling require empirical calibration.

    \item The torsion--semantics correspondence (Open Problem O3) is conjectural. Whether torsion actually encodes directed relations in trained models is an empirical question.

    \item The ``Fiber Bundle Transformer'' is named by structural analogy, not a strict differential-geometric construction (Section~\ref{sec:fiberbundle}). The curvature and torsion objects of Phases B, C are proxies, not the Riemann and torsion tensors of a true connection. A rigorous formulation is left as Open Problem O6.

    \item The metric-preconditioned FFN (Section~\ref{sec:naturalffn}) is a Riemannian-gradient-inspired update, not a natural gradient in the strict Amari sense (which would require a Fisher information metric on parameter space). Moreover, with the chosen parameterization $g_t = I + U_t U_t^\top \succeq I_d$, the metric only stretches directions, so $g_t^{-1} \preceq I_d$ only \emph{down-weights} updates along stretched directions; it never amplifies compressed directions as a general Riemannian-gradient step would. A parameterization admitting eigenvalues below $1$ (e.g.\ $g_t = I + U_t U_t^\top - V_t V_t^\top$ with a PSD projection to ensure positive-definiteness) would restore the full preconditioning effect but at higher computational cost. Whether the representation-space preconditioning accelerates convergence is open (Open Problem O4).

    \item Theorem~\ref{thm:nongram} is a statement about the \emph{functional} (universal) factorization of the Riemannian score, not about the rank of a specific $L \times L$ matrix. For a specific input with $L \leq d$, the score matrix has rank $\leq L = O(d)$ and a $d' = O(d)$ factorization trivially exists; the $\Omega(d^2)$ lower bound is meaningful only in the regime $L = \Omega(d^2)$ or under the functional interpretation (Equation~\eqref{eq:nongrambound}).

    \item The diversity regularizer $\mathcal{L}_{\mathrm{div}}$ (Section~\ref{sec:regularization}) maximizes pairwise metric distance without a principled upper bound on $\|U_t\|_F^2$; the soft cap $\mathcal{L}_{\mathrm{cap}}$ of Equation~\eqref{eq:cap} is a pragmatic patch, and the precise value of $\kappa_{\max}$ relative to the critical threshold $\bar U^2_{\mathrm{crit}} = \Theta(\tau/L)$ requires empirical calibration.
\end{itemize}

\subsection{Conclusion}

The flat Euclidean geometry of Transformer attention is a structural constraint with proven consequences: in pure self-attention stacks, representational rank collapses doubly exponentially with depth. The core mechanism is the row-stochasticity of the attention matrix: after softmax, each row is a probability distribution, and products of such matrices converge to rank-1. Per-token Riemannian metrics, represented as low-rank perturbations $g_t = I + U_t U_t^\top$, change the algebraic structure of the attention scores (Theorem~\ref{thm:nongram}: they cannot be factored as $QK^\top$ with $O(d)$-dimensional factors) but do \emph{not} alter the row-stochasticity of the post-softmax attention matrix. Whether the heterogeneous metric perturbations are sufficient to prevent the row-stochastic contraction, i.e.\ whether they create a persistent spectral gap $\sigma_2(A^{\mathrm{Riem}}) > c > 0$, is the central open problem identified by this framework (Conjecture~\ref{conj:rankstrong}). Section~\ref{sec:counterexample} settles part of this question negatively: metric strength and diversity alone cannot suffice, and any positive resolution must come from directional overlap between the metric directions and the active query--key difference subspace, together with contrastive margins between competing keys (Proposition~\ref{prop:margin}).

The perturbation analysis (Proposition~\ref{prop:perturb}) reveals a self-limiting structure: when the metric correction is small ($\bar U^2 \ll \tau/L$), the Riemannian attention is a perturbation of Euclidean attention and collapse proceeds as before; when the correction is large ($\bar U^2 \gg \tau/L$), the perturbation bound becomes vacuous and standard tools provide no information. The regime where anti-collapse would need to occur is precisely the regime where the current analysis is silent; Section~\ref{sec:counterexample} shows that this silence is not merely an artifact of the method: even at arbitrarily large metric strength, the correction can vanish exactly when the metric directions are orthogonal to the active query--key differences.

On the constructive side, we have established complexity bounds showing that the low-rank parameterization adds $O(d \cdot r)$ precomputation per token for the geodesic distance and $O(d \cdot r^2)$ for the metric inversion via the Woodbury identity (Proposition~\ref{prop:woodbury}), with a score-level overhead of $O(r/d)$; accounting for the metric-generation MLP, the total overhead relative to the attention cost is $O(r/d + dr/L)$ (Proposition~\ref{prop:complexity}); we stress that these are complexity bounds, and that runtime feasibility and training stability at scale remain to be demonstrated empirically. We have also identified a sufficient condition for spectral-gap preservation, a permutation margin on the pre-softmax scores with summable leakage (Proposition~\ref{prop:margin}), and shown in Section~\ref{sec:counterexample} that, absent directional overlap, the metric correction can be exactly inert. The Fiber Bundle Transformer architecture, framed as a structural analogy with explicit approximation-error analysis at each phase, provides a complete specification. The train/inference consistency result (Proposition~\ref{prop:consistency}) establishes that empirical evaluation must match training and inference blending coefficients (or retrain), and the metric collapse analysis (Conjecture~\ref{conj:collapse}) identifies the dominant failure mode requiring architectural countermeasures.

We have been deliberately precise about what is proven (the non-Gram structural property in the non-degenerate regime, the perturbation bound, the orthogonal-subspace obstruction and its explicit counterexample, the permutation-margin sufficient condition for spectral-gap preservation, the local blending-consistency result, and the asymptotic complexity bounds), what is conjectured with explicit conditions (the strengthened rank preservation under directional conditions, the critical-regime transition, the metric collapse failure mode), and what remains open (the super-critical spectral analysis under directional conditions, the differential-geometric formalization, the convergence of the preconditioned FFN, and the empirical question of whether the learned metrics maintain the required directional overlap in practice). The central gap, proving or disproving that heterogeneous Riemannian metrics prevent the rank collapse that row-stochastic attention matrices otherwise cause under the directional conditions identified in Section~\ref{sec:counterexample}, defines the research agenda that this framework equips future work to address.

\section*{Acknowledgments}

The author thanks the open-source community for the tools and mathematical resources that made this work possible, and an independent researcher who communicated the orthogonal-subspace counterexample and the corrected sufficient conditions of Section~\ref{sec:counterexample}; that feedback materially improved the precision of the claims made here.


\end{document}